\documentclass[11pt]{article}

\usepackage{times}
\usepackage{latexsym}
\usepackage[T1]{fontenc}

\usepackage{enumitem}
\usepackage{microtype}
\usepackage{graphicx}
\usepackage{subfigure}
\usepackage{tablefootnote}
\usepackage{multirow}
\usepackage{booktabs} 
\newcommand{\RETURN}{\STATE \textbf{return} }
\usepackage{hyperref}

\def\cX{{\mathcal{X}}}
\def\cD{{\mathcal{D}}}

\def\cY{{\mathcal{Y}}}
\def\RR{{\mathbb{R}}}
\def\EE{{\mathbb{E}}}
\def\la{{\langle}}
\def\ra{{\rangle}}

\usepackage{acl}

\usepackage{amsmath}
\usepackage{amssymb}
\usepackage{mathtools}
\usepackage{amsthm}
\usepackage{xcolor}
\usepackage[table]{xcolor}
\usepackage{algorithm}
\usepackage{algorithmic}

\definecolor{spifcyan}{RGB}{220,245,245}

\DeclareMathOperator*{\argmax}{argmax}
\DeclareMathOperator*{\argmin}{argmin}

\usepackage[capitalize,noabbrev]{cleveref}

\theoremstyle{plain}
\newtheorem{theorem}{Theorem}[section]
\newtheorem{proposition}[theorem]{Proposition}
\newtheorem{lemma}[theorem]{Lemma}

\theoremstyle{definition}
\newtheorem{definition}[theorem]{Definition}
\newtheorem{assumption}[theorem]{Assumption}
\theoremstyle{remark}
\newtheorem{remark}[theorem]{Remark}

\usepackage[textsize=tiny]{todonotes}
\title{Your Self-Play Algorithm is Secretly an Adversarial Imitator:\\
Understanding LLM Self-Play through the Lens of Imitation Learning}

\author{
Shangzhe Li$^{1}$ \quad
Xuchao Zhang$^{2}$ \quad
Chetan Bansal$^{2}$ \quad
Weitong Zhang$^{1}$ \\
$^{1}$University of North Carolina at Chapel Hill \\
$^{2}$Microsoft Research \\
\texttt{\{shangzhe, weitongz\}@unc.edu} \\
\texttt{\{xuchaozhang, chetanb\}@microsoft.com}
}

\begin{document}

\maketitle

\begin{abstract}
Self-play post-training methods has emerged as an effective approach for finetuning large language models and turn the weak language model into strong language model without preference data. 
However, the theoretical foundations for self-play finetuning remain underexplored. 
In this work, we tackle this by connecting self-play finetuning with adversarial imitation learning by formulating finetuning procedure as a min–max game between the model and a regularized implicit reward player parameterized by the model itself. 
This perspective unifies self-play imitation and general preference alignment within a common framework.
Under this formulation, we present a game-theoretic analysis showing that the self-play finetuning will converge to it's equilibrium.
Guided by this theoretical formulation, we propose a new self-play imitation finetuning algorithm based on the $\chi^2$-divergence variational objective with bounded rewards and improved stability. 
Experiments on various of language model finetuning tasks demonstrate consistent improvements over existing self-play methods and validate our theoretical insights.
\end{abstract}

\section{Introduction}
\label{sec:intro}
Large language models (LLMs) have demonstrated remarkable success across a wide range of applications that require complex reasoning or specialized domain knowledge. A major recent advance in LLM development is post-training alignment toward more desirable behaviors~\citep{mishra2022cross,thoppilan2022lamda,chung2024scaling}. Modern post-training pipelines typically combine supervised finetuning (SFT)~\citep{ouyang2022training} with a variety of reinforcement learning from human feedback (RLHF) methods~\citep{bai2022training,rafailov2023direct,guo2025deepseek}.

Among these methods, Many reinforcement learning (RL)-based finetuning approaches~\citep{rafailov2023direct, wu2024self,zhang2024iterative,calandriello2024human,zhang2025improving} rely on a (general) preference oracle to label samples, which encourage the model to learn from preferred responses over the undesirable ones. To reduce reliance on human preference annotations, recent methods \citep{chen2024self,wang2026space,wang2026triplets} study a \textit{self-play} regime and treat the ground-truth responses as positive samples and the model-generated responses as negative counterparts. Despite recent empirical advancement, the theoretical understanding of the self-play regime remains limited.

This gap in theory limits principled development of self-play algorithms. For example, considering a nontrivial subset of prompts in the dataset, the ground-truth response can be closely comparable in quality, or even partially worse than, the model-generated response; in such cases, the induced preference signal becomes ambiguous or misspecified, making the implicit reward model prone to overfitting these irregular comparisons and thus failing to provide a reliable learning signal. In such a cases, existing self-play formulations lack a clear theoretical mechanism to regularize the reward model and can introduce misspecified supervision and destabilize training. 
This raises the following question: 
\begin{center}\textit{
How can we theoretically characterize the (implicit) reward learning in self-play finetuning?}
\end{center}
To answer this question, we connect the self-play finetuning with the adversarial imitation learning (AIL) framework. 
Serving long as a principled framework for imitation and inverse reinforcement learning~\citep{ho2016generative,abbeel2004apprenticeship}, AIL formulates imitation from expert demonstrations as a two-player game between a policy learner and a reward learner where the reward learner aims to distinguish expert behavior from learner behavior robustly, while the policy learner seeks to match the expert distribution by maximizing the reward model. 
This paradigm has been successfully applied to a variety of robotics tasks~\citep{rafailov2021visual,ablett2023learning}. Notably, a series of work~\citet{garg2021iq,al2023ls,ren2024hybrid} has considered regularizing the reward learning with improved empirical performance for classical RL tasks.

In this work, we establish a conceptual and algorithmic connection between adversarial imitation learning and self-play finetuning in large language models. We formulate this alignment process as a min–max game in which the policy player corresponds to the language model itself, while the reward player can be reparameterized using the model and its previous snapshots. Based on this formulation, our contributions are threefold:
\begin{itemize}[leftmargin=*, topsep=0pt,itemsep=0pt,partopsep=0pt,parsep=0pt]
\item We establish an adversarial imitation learning–based framework for self-play imitation finetuning of large language models, which naturally generalizes to self-play methods with general preference alignment.
\item We provide a rigorous game-theoretic analysis of self-play language model finetuning within the adversarial imitation learning framework. Guided by this analysis, we propose a novel self-play imitation finetuning algorithm with theoretical advantages over existing approaches.
\item We empirically evaluate our method on various families of language models, demonstrating consistent performance improvements over prior methods and validating our theoretical insights especially a more robust reward learner.
\end{itemize}

\section{Related Works}
\label{sec:related-works}
In this section, we discuss the related works on the imitation learning and self-play finetuning of language models. 

\noindent \textbf{Imitation Learning.} Imitation learning is an variation of reinforcement learning where the agent aims to \textit{imitate} expert behavior leveraging the reward-free expert demonstration. Historically, imitation learning approaches can be broadly categorized into two major classes. The first category is usually referred to as \textit{behavioral cloning}~\citep{florence2022implicit,chi2024diffusionpolicy} which directly imitate the expert demonstrations in a supervised learning manner. Several recent works have established theoretical guarantees under these settings \citep{foster2024behavior,rohatgi2025computational}.

The second category adopts a variational formulation casting imitation learning as a min-max optimization between a reward model differentiating agent's behavior from expert demonstration and a policy optimization trying to maximize the reward. This approach are usually referred to as adversarial imitation learning (AIL). Representative methods include GAIL \citep{ho2016generative}, IQ-Learn \citep{garg2021iq}, and LS-IQ \citep{al2023ls}. This line of work has also been supported by rigorous theoretical analyses \citep{liu2021provably,xu2024provably,li2025near}.
\begin{figure*}
    \centering
    \includegraphics[width=0.9\linewidth]{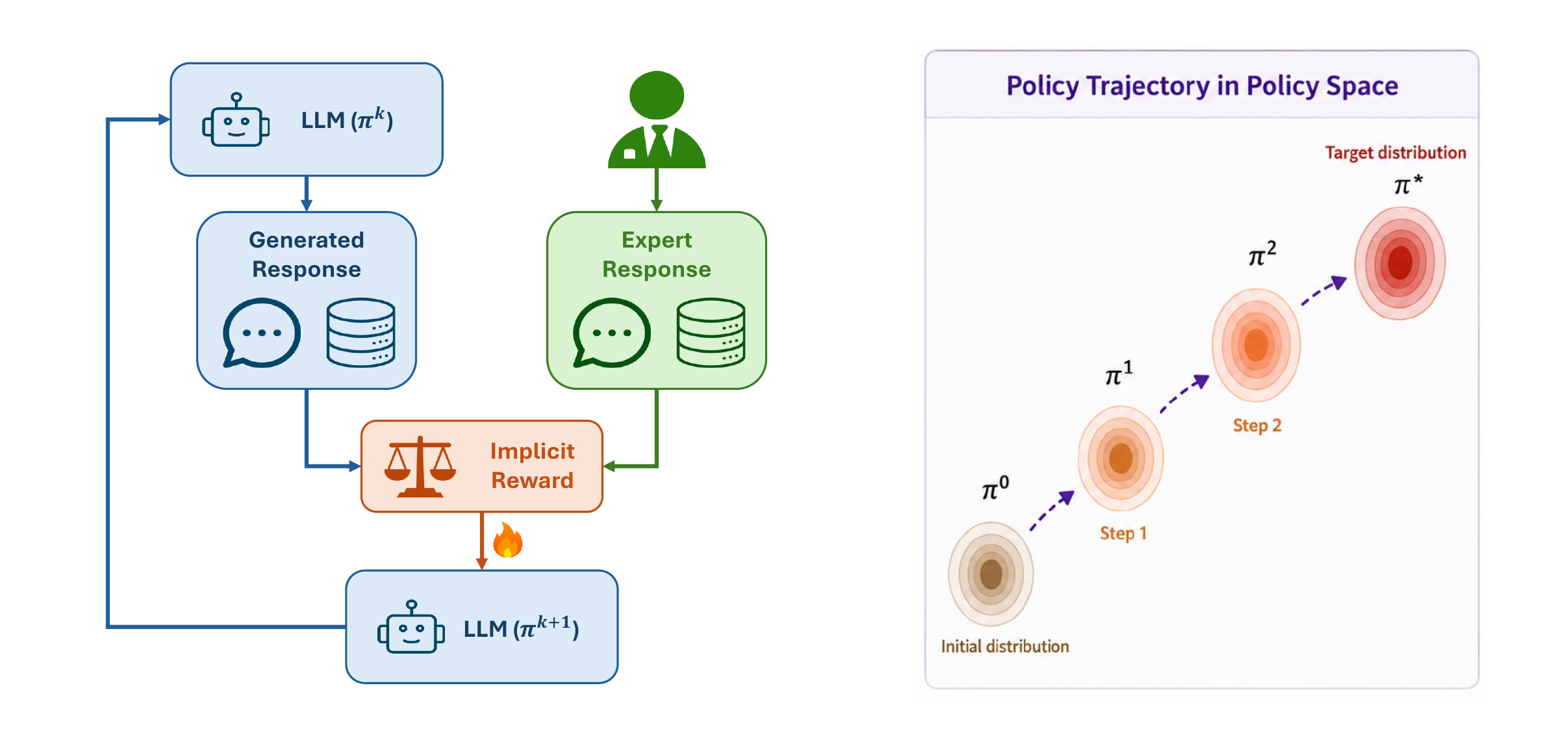}
    \caption{\textbf{Illustration of the LLM Self-Play Imitation Setup.}
We illustrate the LLM self-play imitation setup studied in our paper, following prior works \citep{chen2024self,wang2026space,wang2026triplets}. Given a batch of SFT data sampled from an expert policy $\pi^\star$, the language model is trained to align with $\pi^\star$ through iterative refinement. At each iteration, an implicit reward model is constructed by reparameterizing the LLM, which learns to discriminate between expert data and generated data from the previous policy $\pi^k$.}
\vspace{-15pt}
\label{fig:pipeline}
\end{figure*}

\noindent \textbf{Self-play with Preference Feedback.} Self-play algorithms have been extensively studied in the context of large language model alignment. Many existing approaches focus on general preference alignment, where the model is iteratively updated using samples labeled by a preference oracle, often inspired by the DPO framework~\cite{rafailov2023direct}. This line of work is frequently described as RL with AI feedback, in which the preference oracle is typically instantiated by an LLM. For example, CPL~\citep{hejna2023contrastive} optimizes policies from preference data via contrastive learning; iterative-DPO~\citep{tu2025enhancing,wu2024self} repeatedly relabels model-generated responses with a preference oracle and applies DPO-style updates; and $\chi$PO~\citep{huang2024correcting} introduces $\chi^2$-based regularization to stabilize policy optimization. In addition, SPPO~\citep{wu2024self} develops a preference-based self-play framework with a squared-loss objective and formulates the interaction as a constant-sum two-player game for general preference model. Closely related works~\citep{calandriello2024human,zhang2024iterative,zhang2025improving} further cast preference alignment as seeking a Nash equilibrium in two-player games, yielding both improved empirical performance and stronger theoretical guarantees.

\noindent \textbf{Self-Play Finetuning with SFT Datasets.} In parallel to self-play with a preference oracle, another line of work studies self-play finetuning using standard SFT datasets that contain expert (ground-truth) demonstrations. For instance, SPIN~\citep{chen2024self} introduces a DPO-style self-play objective that enables the model to iteratively imitate SFT data by competing against its own past instances. While SPIN does not require an explicit preference oracle, it relies on expert prompt–response pairs from SFT as the self-play targets. We refer to these methods as \textit{self-play imitation}, as it directly imitates SFT behavior rather than indirectly aligning to a learned or external preference oracle. We provide an illustration of this setup in Figure~\ref{fig:pipeline}. Follow-up works in a similar setting include SPACE \citep{wang2026space} and T-SPIN \citep{wang2026triplets}, which further improve upon SPIN.

Our work lies at the intersection of imitation learning and large language model self-play. We reinterpret self-play imitation finetuning methods, such as SPIN~\citep{chen2024self}, through the lens of adversarial imitation learning, and provide the first unified game-theoretic analysis of this class of methods. We further show that the same perspective naturally extends to self-play algorithms for general preference, including SPPO~\citep{wu2024self} and INPO~\citep{zhang2024iterative}. Finally, by instantiating our framework with a variational formulation based on the $\chi^2$ divergence, in the spirit of IQ-Learn~\citep{garg2021iq} and LS-IQ~\citep{al2023ls}, we derive a self-play imitation finetuning algorithm that is both more stable and empirically more effective. Notably, Appendix~\ref{sec:discussion} provides a unified and rigorous discussion of both self-play imitation and preference-based self-play finetuning under our AIL formulation.

\section{Preliminaries}
\label{sec:prelim}
We formulate the language model generation process as a contextual bandit problem. For each round, the language model observes a context $x\in\cX$ and generates a response $y\in\cY$. There exists a reward function $r(x,y)$ in a reward class $\mathcal{R}$ can be learned for each context and response pairs. We represent the language model as a policy $\pi(y|x)$ in a policy class $\Pi$. Since the model aims to align with a domain which an inaccessible generation policy $\pi^\star$ is preferred, the policy learning process can be formulated as an adversarial process $\max_r\min_{\pi\in\Pi}~\EE_{\pi^\star}[r]-\EE_\pi[r]$ that jointly learn the reward and policy. This formulation is consistent with the standard adversarial imitation learning setup.

\noindent \textbf{Self-play finetuning.} Self-play finetuning (SPIN) \citep{chen2024self} aims to align a language model with an SFT dataset $\mathcal{D}^\star$ generated by an expert policy $\pi^\star(y\mid x)$, which may represent human behavior. At each iteration, SPIN updates the model by maximizing the following objective:
\begin{align*} 
\underset{\substack{(x,y)\sim\mathcal{D}^\star\\ y' \sim \pi^k(\cdot | x)}}{\mathbb{E}}\!\left[\sigma\left(\beta\log\frac{\pi(y|x)}{\pi^k(y|x)}\!-\!\beta\log\frac{\pi(y'|x)}{\pi^k(y'|x)}\right)\right], 
\end{align*}
where $\pi^k$ denotes the model from the previous iteration, 
$\sigma(\cdot)$ is a monotonically non-decreasing link function. 
By iteratively optimizing this objective and resampling data, SPIN drives the policy $\pi$ toward the expert policy $\pi^\star$.

\noindent \textbf{Adversarial Imitation Learning.} Instead of directly mimics the expert demonstrations using behavioral cloning,
adversarial imitation learning (AIL) formulates the learning process as a two-player game, providing a variational characterization of distribution matching between the expert and behavioral policies. Formally, this involves jointly learning the reward and policy via the following optimization:
\begin{align*}
\max_r\min_\pi\underset{(s,a)\sim d^\star}{\mathbb{E}}[r(s,a)] -\!\!\!\!\!\underset{(s,a)\sim d^\pi}{\mathbb{E}}[r(s,a)] -\psi(r),
\end{align*}
where $d^\pi = (1-\gamma)\sum_{t=0}^{\infty}\gamma^t \Pr_\pi(s_t = s,\, a_t = a)$ denotes the discounted occupancy measure induced by the behavioral policy $\pi$, and $d^\star$ is defined analogously for the expert policy $\pi^\star$. $\psi(r)$ is a convex regularizer associated with the choice of statistical distance between the expert and behavioral occupancy measures \citep{garg2021iq}. Although we formulate adversarial imitation learning (AIL) within the Markov Decision Process (MDP) framework, in the LLM setting considered in this work the formulation naturally reduces to a contextual bandit problem, since no meaningful transition dynamics are present.

\looseness=-1

For the simplicity of the theoretical analysis, we assume that the optimal expert policy as well as the ground truth reward is realizable. 
\begin{assumption}[Realizability]
\label{ass:realize}
The ground-truth reward and optimal policy are realizable with the corresponding function classes, i.e., $r^\star{\in}\mathcal{R}$, $\pi^\star{\in}\Pi$.
\end{assumption}

\section{Adversarial Imitation Learning View}

\label{sec:ail-view}
\subsection{A Single-Stage Formulation}
\label{sec:single-stage}
In this section, we focus on formulating the self-play finetuning problem imitating an expert data distribution $\pi^\star$ from an initial policy distribution as an adversarial imitation learning process:
\begin{align}
\label{eqn:ail-formulation}
&\max_{r} \min_{\pi}\EE_x\Big[\sigma\big(\!\!\!\!\!\underset{\substack{\,\\[.01em]\;\;\;\;y{\small \sim}\pi^\star}}{\EE}\!\!\!\![r(x,y)]{-}\!\!\!\!\underset{\substack{\,\\[.01em]\;\;\;\;y{\small \sim}\pi}}{\EE}\!\!\!\![r(x,y)]\big)\!\!-\!\!\psi(r)\Big],
\end{align}
where $\sigma$ denotes the monotonically non-decreasing link function such as $\sigma(t)=t$ or $\sigma(t)=-\log(1+\exp(-t))$, and $\psi(r)$ is a convex regularizer \citep{ho2016generative}. Notably, optimizing the objective in~\eqref{eqn:ail-formulation} is equivalent to minimizing a statistical distance between the expert and current policy distributions \citep{ho2016generative,garg2021iq}, given a identical link function $\sigma(t)$.

In this work, we focus on the choice of the convex regularizer $\psi(r)$, corresponding to the original regularization term $\phi(r,r^{k-1})$ with the Bregman divergence component removed. Different choices of $\psi(r)$ induce different properties in the resulting self-play algorithms. We summarize the connection between the choice of the regularizer and the resulting self-play algorithm in Table~\ref{tab:divergences}. In particular, we study two specific regularizer choices that cast the min–max optimization in~\eqref{eqn:ail-formulation} as a statistical distance minimization:
\begin{table*}[t]
    \centering
    \resizebox{\linewidth}{!}{
    \begin{tabular}{c|c|c|c|c}
    \toprule
        Imitation Objective & Algorithm & $\sigma(t)$ & $\psi(r)$ & Distance \\
    \midrule
        \multirow{3}{*}{Query Response Pairs}
        &SPIN~\citep{chen2024self} & $-\log(1+e^{-t})$ & $\psi(r)=\infty \cdot \mathbf{1}[|r|_{\infty} > R_{\max}]$ & $D_{\text{KL}}$\\
        & (Linear) SPIN & $t$ & $\psi(r)=\infty \cdot \mathbf{1}[|r|_{\infty} > R_{\max}]$ & $D_{\text{TV}}$\\
        & SPIF (Ours) & $t$ & $\psi(r)=\mathbb{E}_{\text{mix}}[r^2]$ & $D_{\chi^2,\text{mix}}$\\
    \midrule
        \multirow{3}{*}{Preference Oracle}
        & Iter-DPO~\citep{tu2025enhancing} & $-\log(1+e^{-t})$ & $\psi(r)=\infty \cdot \mathbf{1}[|r|_{\infty} > R_{\max}]$ & $D_{\text{KL}}$ \\
        & SPPO~\citep{wu2024self} & $t$ & $\psi(r)=\mathbb{E}_{\text{mix}}[r^2]$ & $D_{\chi^2,\text{mix}}$ \\
        & INPO~\cite{zhang2025improving} & $t$ & $\psi(r)=\mathbb{E}_{\text{mix}}[r^2]$ & $D_{\chi^2,\text{mix}}$ \\
    \midrule
        \multirow{3}{*}{Expert Trajectories (MDPs)}
        & GAIL~\cite{ho2016generative} & $t$ & $\psi(r)=\mathbb{E}[r-\log(2-e^{-r})]$ & $D_{\text{JS}}$ \\
        & IQ-Learn~\cite{garg2021iq} & $t$ & $\psi(r)=\mathbb{E}[r^2]$ & $D_{\chi^2}$ \\
        & LS-IQ~\cite{al2023ls} & $t$ & $\psi(r)=\mathbb{E}_{\text{mix}}[r^2]$ & $D_{\chi^2,\text{mix}}$ \\
    \bottomrule
    \end{tabular}
    }
    \caption{\textbf{Overview of the Algorithms.} This table summarizes AIL and LLM self-play algorithms under our unified framework, with different learning setting, imitation objective, choice of link function and convex regularizer, and the resulting statistical distance being minimized for each algorithm. $\mathbb{E}_{\text{mix}}$ denotes the mixed regularizer $\psi(r)=\tfrac{c}{2}\mathbb{E}_{\pi^\star}[r(x,y)^2]+\tfrac{c}{2}\mathbb{E}_{\pi}[r(x,y)^2]$, $D_{\chi^2,\text{mix}}$ denotes the mixed $\chi^2$ divergence, defined between expert and model data for non-preference-based methods, and between positive and negative samples for preference-based methods. \textit{(Linear) SPIN} is refer to as a variant of SPIN~\citep{chen2024self} using identical link function $\sigma(t) = t$. INPO is an KL-constrained optimization w.r.t $\pi^{\mathrm{ref}}$, so it has an additional regularizer compared to other methods.}
    \label{tab:divergences}
    \vspace{-15pt}
\end{table*}

\noindent \textbf{Total Variation Distance.} Consider a regularizer with $\psi(r)=0$ for $\Vert r\Vert_\infty\leq R_{\max}$ and $\psi(r)=\infty$ otherwise, and an identical link function with $\sigma(t)=t$. This formulation is equivalent to $\min_{\pi\in\Pi} R_{\max} D_{\text{TV}}(\pi,\pi^\star)$ under a trust region constraint. Taking the closed form of the policy update, can recover the learning objective of (Linear) SPIN~\citep{chen2024self}. However, original the reward reparameterization used in SPIN doesn't explicitly enforce the boundedness of the reward, thus $R_{\max}$ can be arbitrarily large. \looseness=-1

\noindent \textbf{KL Divergence.} For $\psi(r) =\infty \cdot \mathbf{1}[|r|_{\infty} > R_{\max}]$, and a link function with $\sigma(t)=-\log(1+\exp(-t))$. This formulation can be seen as minimizing the KL divergence, which recovers SPIN \citep{chen2024self} under logistic link function. We prove it under a multi-iteration two-stage surrogate, as used in SPIN, in Appendix~\ref{sec:non-linear-spin}. Notably, $R_{\max}$ here is unbounded and can be arbitrary large.

\noindent \textbf{Pearson $\chi^2$ Divergence.} Consider a regularizer with $\psi(r)=c\EE_{\pi^\star}[(r(x,y))^2]$ and a link function $\sigma(t)=t$. This formulation is equivalent to $\min_{\pi\in\Pi} D_{\chi^2}(\pi\Vert\pi^\star)$ under a trust region constraint. This regularizer can further be generalized to $\psi(r)=c\alpha\cdot\EE_{\pi^\star}[(r(x,y))^2]+c(1-\alpha)\cdot\EE_{\pi}[(r(x,y))^2]$, which can characterize the data mixing strategy during self-play and equivalent to $\min_{\pi\in\Pi} D_{\chi^2}(\pi^\star\Vert\alpha\pi^\star+(1-\alpha)\pi)$ \citep{al2023ls} under trust region constraints on policy and reward. Since SPIN with both identical and sigmoid link function cannot guarantee a bounded reward, the magnitude of the reward from SPIN may be arbitrarily large in their cases. In contrast, we can show that by applying the equally sampled mixed $\chi^2$ regularization, i.e., $\psi(r)=(1/2)c\cdot\EE_{\pi^\star}[(r(x,y))^2]+(1/2)c\cdot\EE_{\pi}[(r(x,y))^2]$, can lead to a bounded divergence and a bounded reward, which leads to theoretical benefits in the following analysis in the two-stage optimization alternative. Notably, \citet{al2023ls} has derived a similar result within MDP setting with occupancy distribution matching. We'll adapt their result to contextual bandit setting in our case. We now introduce the following proposition showing the boundedness of the reward when using the mixed $\chi^2$ convex regularizer:
\begin{proposition}[Contextual bandit version of Proposition A.2 and A.3, \citealt{al2023ls}]
\label{prop:bounded-reward}
The mixed Pearson $\chi^2$ divergence between $\pi^\star$ and the mixture distribution $\tfrac{\pi+\pi^\star}{2}$ induced by the convex regularizer $\psi(r)=\tfrac{c}{2} \cdot\left(\EE_{\pi^\star}[(r(x,y))^2] + \EE_{\pi}[(r(x,y))^2]\right)$ is bounded by:
\begin{align*}
0\leq 2D_{\chi^2}\left(\pi^\star\| \tfrac{\pi+\pi^\star}{2}\right)\leq\tfrac{1}{c}.
\end{align*}
Furthermore, the optimal reward for solving the variational form of this Pearson $\chi^2$ divergence:
\begin{align*}
&2D_{\chi^2}\big(\pi^\star\|\tfrac{\pi+\pi^\star}{2}\big)\\
&=\sup_r\EE_{\pi^\star}[r(x,y)]-\EE_{\pi}[r(x,y)]\\
&\qquad- \frac{c}{2}\left(\EE_{\pi^\star}[(r(x,y))^2] + \EE_{\pi}[(r(x,y))^2]\right)
\end{align*}
is bounded within the interval $[-\tfrac{1}{c}, \tfrac{1}{c}]$.
\end{proposition}
\subsection{A Two-Stage Iterative Alternative}
Directly solving the min-max optimization problem in~\eqref{eqn:ail-formulation} may be hard in practice, prior work usually decompose it into a two-stage iterative optimization process \citep{ho2016generative,garg2021iq}. This iterative process can be decomposed into a two-stage algorithm iteratively optimizing the reward $r$ and policy $\pi$:
\begin{align}
\label{eqn:two-stage-formulation}
r^k&\!=\argmax_r\!\EE_{\rho}\big[\sigma(\EE_{\pi^\star}[r(x,y)]\!-\!\EE_{\pi^{k}}[r(x,y)]) \notag \\
&\qquad\qquad\qquad -\phi(r,r^{k-1})\big],\notag \\
\pi^{k+1}&\!=\argmax_\pi \EE_{\rho}\left[\EE_{\pi}r^k(x,y) - \beta D_{\text{KL}}(\pi\Vert\pi^{k})\right].
\end{align}
To attain a more stable optimization process, we seek to constrain the optimization in~\eqref{eqn:two-stage-formulation} by regulating the optimization process with an additional KL constraint $D_{\text{KL}}(\pi\Vert\pi^{k})$ on the policy for mirror descent and a one-step Bregman constraint for $r$, which leads to a new convex regularizer $\phi(r,r^{k-1})=\psi(r)+\zeta D_f(r,r^{k-1})$. $D_f(r,r^{k-1})$ is a Bregman divergence defined by a convex function $f$. 
One thing worth mentioning is that the policy optimization objective has a closed-form solution $\pi^{k+1}\propto\pi^{k}\exp(\beta^{-1}r^k)$. SPIN \citep{chen2024self} leveraged this closed-form to reparameterize the reward function, converting the two-stage algorithm into a single-stage iterative policy optimization.

\begin{algorithm}[t]
\caption{Self-Play Imitation Finetuning (General)}
\begin{algorithmic}[1]
\label{alg:self-play-imitation}
\STATE \textbf{Input:} Number of iterations $K$, expert policy $\pi^\star$, reference policy $\pi^{\text{ref}}$.
\STATE Initialize $\pi^1=\pi^{\text{ref}}$.
\FOR{$k = 1, 2, \dots, K$}
\STATE Obtain $r^k$ by solving\\  $\textstyle{\argmax_r}\EE_{\pi^\star}[r(x,y)]-\EE_{\pi^{k}}[r(x,y')]-\phi(r,r^{k-1})$
\STATE Update policy $\pi^{k+1}$ by solving\\ $\textstyle{\argmax_\pi}\EE_{\pi}[r^k(x,y)]-\beta D_{\text{KL}}(\pi\Vert\pi^{k})$
\ENDFOR
\RETURN $\pi^{K+1}$
\end{algorithmic}
\end{algorithm}

\subsection{A Game-Theoretic Analysis}
\label{sec:theoretical-analysis}
We first formulate~\eqref{eqn:ail-formulation} as a two-player game between policy $\pi$ and reward $r$. Thus the following weak duality holds:

\begin{proposition}
\label{prop:ne}
The problem defined in~\eqref{eqn:ail-formulation} has the following weak duality with a identical link function $\sigma(\cdot)$:
\begin{align*}
\min_{\pi} J(\pi,r^\star)\leq J(\pi^\star,r^\star)\leq\max_{r} J(\pi^\star,r),
\end{align*}
where $J(\pi,r)=\underset{x\sim\rho(x)}{\EE}\!\!\big[\underset{y\sim\pi^\star}{\EE}r(x,y)-\underset{y\sim\pi}{\EE}r(x,y)\big]$.
\end{proposition}
Following Proposition~\ref{prop:ne}, we further characterize the duality gap of the alternative two-stage adversarial algorithm:
\begin{definition}
\label{def:dual-gap}
For Algorithm~\ref{alg:self-play-imitation}, define the duality gap as
\begin{align*}\textstyle{
\mathrm{DualGap}=\max_{r\in\mathcal{R}}{J}(\bar\pi,r)-\min_{\pi\in\Pi}{J}(\pi,\bar r),
}\end{align*}
where $\bar r= \tfrac1K\sum_{k=1}^K r^k$ and $\bar \pi=\tfrac1K\sum_{k=1}^K \pi^k$.
\end{definition}
By Definition~\ref{def:dual-gap} and Proposition~\ref{prop:ne}, we are ready to provide an upper bound for the duality gap of the self-play imitation algorithm:
\begin{theorem}
\label{thm:upper-dual}
Let $r \in \mathcal \{\cX{\times}\cY{\rightarrow}[\text{-}R_{\max},R_{\max}]\} \cap \mathcal R$, and set constants $D = \max_{\pi\in\Pi}\!D_{\text{KL}}(\pi^\star\Vert\pi)$, $B {=} \max_{r\in\mathcal{R}}\!D_{f}(r^\star,r) /R_{\max}^2$, and parameter $\zeta = \tfrac{\sqrt{K}}{BR_{\max}^2}, \beta{=}\tfrac{\sqrt{K}}{D}$ in Algorithm~\ref{alg:self-play-imitation}. When an identical link function $\sigma(t){=}t$ is applied, the duality gap is upper bounded by:
\begin{align*}
&\mathrm{DualGap}\!\leq\!\mathcal{O}\bigg(\frac{(D+B)R^2_{\max}}{\sqrt{K}}\bigg).
\end{align*}
\end{theorem}
\begin{remark}
    Theorem~\ref{thm:upper-dual} suggests the upper bounds for the iteration steps $K\leq\mathcal{O}(R^4_{\max}(D+B)^2\epsilon^{-2})$. Similar setting (adversarial imitation formulation with KL-constrained policy update) have been studied in OGAP \citep{liu2021provably} but with linear MDP and without estimation error. They achieve $\tilde{\mathcal{O}}(1/\sqrt{K})$ suboptimality gap upper bound, which matches our results although with different setting. Prior work considering the setting of Nash policy optimization with general preference also achieves the upper bound for $K$ with similar order \citep{zhang2024iterative}.
\end{remark}
\begin{remark}
    From Proposition~\ref{prop:bounded-reward}, we have that using the mixed Pearson $\chi^2$ divergence as the choice for the convex regularizer leads to a bounded reward, where $R_{\max}=1/c$. This suggests the theoretical benefit of leveraging $\chi^2$ divergence as the regularization compared to using KL divergence or TV distance as in SPIN \citep{chen2024self}, since a bounded reward may result in small $R_{\max}$, and a tighter upper bound for the duality gap according to Theorem~\ref{thm:upper-dual}.
\end{remark}
\begin{remark}
    The reward space $\mathcal{R}$ in Theorem~\ref{thm:upper-dual} is defined by strictly enforcing the regularizer $\psi(r)$, which can be seen as turning the Lagrangian dual form in~\eqref{eqn:ail-formulation} into a constrained optimization by turning $\psi(r)$ from to regularizer to a hard constraint on the reward space. By employing Assumption~\ref{ass:realize}, we assume that the optimal reward is still in the constrained reward space. 
\end{remark}

Following Theorem~\ref{thm:upper-dual}, leveraging a mixed $\chi^2$ regularizer $\psi(r)=(1/2)c\cdot\EE_{\pi^\star}[(r(x,y))^2]+(1/2)c\cdot\EE_{\pi}[(r(x,y))^2]$, by Proposition~\ref{prop:bounded-reward} we can have the bounded reward property, i.e., the reward that solves the variational form of Pearson $\chi^2$ divergence is bounded by $[-1/c,1/c]$. In this case, we can have a bounded $R_{\max}$, which leads to a tighter upper bound depicted in Theorem~\ref{thm:upper-dual}, compared to unbounded formulation in SPIN \citep{chen2024self}.

\section{$\chi^2$ Self-Play Imitation Finetuning}
We consider a setting that given a finetuning dataset $\cD^\star$ containing query-response pairs sampled from an oracle $\pi^\star$ similar with SPIN \citep{chen2024self}. Since Proposition~\ref{prop:bounded-reward} and Theorem~\ref{thm:upper-dual} has shown theoretical benefit of leveraging $\chi^2$ divergence as the convex regularizer in the self-play imitation setting. We aim to derive a practical algorithm under the scope of using $\chi^2$ divergence by choosing a identical link function $\sigma(t)=t$ and a convex regularizer $\psi(r)=\alpha\cdot\EE_{\pi^\star}[(r(x,y))^2]+(1-\alpha)\cdot\EE_{\pi}[(r(x,y))^2]$ for the formulation in~\eqref{eqn:two-stage-formulation}. The KL-regularized policy optimization~\eqref{eqn:two-stage-formulation} yields the following closed-form solution:
\begin{align*}
    r(x,y)=\beta\log\left(\frac{\pi(y|x)}{\pi^{k}(y|x)}\right)+\beta \log Z(x).
\end{align*}
To avoid estimating the partition function $Z(x)$, we establish a reparameterization of the reward by subtracting the partition function:
\begin{align}
\label{eqn:reward-mapping}
    r(x,y):=\beta\log\!\left(\frac{\pi(y|x)}{\pi^{k}(y|x)}\right)\!.
\end{align}
Using this reward reparameterization, we can turn the $\chi^2$ regularized two-stage formulation in~\eqref{eqn:two-stage-formulation} into a singe stage least square optimization problem in the following proposition:
\begin{table*}[t]
\centering
\scriptsize
\setlength{\tabcolsep}{2.0pt}
\renewcommand{\arraystretch}{1.05}
\resizebox{\textwidth}{!}{
\begin{tabular}{@{}l|cccc|cccc|cccc@{}}
\toprule
& \multicolumn{4}{c|}{\textbf{Qwen3-4B}}
& \multicolumn{4}{c|}{\textbf{Mistral-7B}}
& \multicolumn{4}{c}{\textbf{Qwen3-14B}} \\
\textbf{Methods}
& \textbf{Arc} & \textbf{MMLU} & \textbf{HS} & \textbf{WG}
& \textbf{Arc} & \textbf{MMLU} & \textbf{HS} & \textbf{WG}
& \textbf{Arc} & \textbf{MMLU} & \textbf{HS} & \textbf{WG} \\
\midrule
Base
& 51.62 & 68.33 & 67.57 & 65.43
& 53.33 & 53.18 & 74.47 & 71.67
& 58.70 & 77.97 & 81.34 & 73.64 \\
SFT
& 54.46 & 68.58 & 69.68 & 66.80
& 54.24 & 54.08 & 76.11 & 72.77
& 59.02 & 77.76 & 82.54 & 73.82 \\
\midrule
SPIN Iter-1
& 53.84 & 68.10 & 67.98 & 66.06
& \textbf{54.21} & 54.11 & 75.52 & 72.93
& 59.24 & 77.36 & 82.11 & 73.18 \\
SPACE Iter-1
& 54.62 & 68.45 & 68.36 & 66.37
& 53.86 & 54.43 & \textbf{76.29} & 72.94
& 59.43 & 77.82 & 81.93 & \textbf{73.65} \\
\rowcolor{spifcyan}
SPIF Iter-1 (Ours)
& \textbf{55.12} & \textbf{68.66} & \textbf{70.48} & \textbf{68.11}
& 54.05 & \textbf{54.60} & 75.92 & \textbf{73.09}
& \textbf{59.71} & \textbf{78.05} & \textbf{82.45} & 73.52 \\
\midrule
SPIN Iter-2
& 54.58 & 67.90 & 68.10 & 67.13
& \textbf{54.52} & 54.40 & 75.39 & 73.14
& 59.43 & 77.84 & 82.70 & 73.78 \\
SPACE Iter-2
& 55.37 & 68.24 & 69.86 & 67.05
& 54.18 & 54.81 & 76.47 & 73.21
& 59.74 & 77.65 & 82.93 & 73.80 \\
\rowcolor{spifcyan}
SPIF Iter-2 (Ours)
& \textbf{56.66} & \textbf{68.75} & \textbf{71.43} & \textbf{68.34}
& 54.41 & \textbf{55.02} & \textbf{76.63} & \textbf{74.33}
& \textbf{60.02} & \textbf{78.12} & \textbf{83.04} & \textbf{74.05} \\
\midrule
SPIN Iter-3
& 55.12 & 68.06 & 68.32 & 67.61
& 54.60 & 54.38 & 75.47 & 73.42
& 59.80 & 77.96 & 82.62 & 73.02 \\
SPACE Iter-3
& 56.84 & 68.16 & 70.34 & 67.45
& 54.16 & 55.02 & 76.60 & 73.45
& 59.93 & 77.81 & 83.08 & 73.64 \\
\rowcolor{spifcyan}
SPIF Iter-3 (Ours)
& \textbf{57.11} & \textbf{68.83} & \textbf{71.92} & \textbf{68.82}
& \textbf{54.70} & \textbf{55.24} & \textbf{77.14} & \textbf{75.02}
& \textbf{60.23} & \textbf{78.08} & \textbf{83.33} & \textbf{74.43} \\
\bottomrule
\end{tabular}
}
\caption{\textbf{Main Results.}
We report results over three iterations and five random seeds on three language models, comparing our method against the supervised finetuning (SFT) baseline, SPIN~\citep{chen2024self}, and SPACE \citep{wang2026space}.
Our approach consistently outperforms the existing baselines across most settings.
}
\label{tab:main-results}
\vspace{-15pt}
\end{table*}
\begin{proposition}
Updating via~\eqref{eqn:two-stage-formulation} under the mapped reward defined in~\eqref{eqn:reward-mapping} and the regularizer in the form of $\psi(r)=c\alpha\cdot\EE_{\pi^\star}[(r(x,y))^2]+c(1-\alpha)\cdot\EE_{\pi}[(r(x,y))^2]$ is equivalent to minimizing the following regularized objective:
\begin{align*}
&\pi^{k+1}=\textstyle{\argmin_{\pi}}~\mathcal{L}(\pi)+\EE_{\rho,\pi^\star,\pi^k}~\zeta D_{f}(\pi,\pi^{k}),
\end{align*}
where $\mathcal{L}(\pi)$ is the least square objective:
\begin{align*}
\mathcal{L}(\pi)&:=\alpha~\mathbb{E}_{\rho,\pi^{\star}(x)}\Big[\beta\log\frac{\pi(y|x)}{\pi^{k}(y|x)}-r_{\max}\Big]^2\\
&+\!(1-\alpha)~\mathbb{E}_{\rho, \pi^{k}(x)}\Big[\beta\log\frac{\pi(y|x)}{\pi^{k}(y|x)}-r_{\min}\Big]^2,
\end{align*}
in which $r_{\max}=1/(2c\alpha)$, $r_{\min}=-1/[2c(1-\alpha)]$ and $D_{f}(\pi,\pi^k)=(1/2)\EE_{\pi^\star(x)}[\beta\log(\pi(y|x)/\pi^{k}(y|x))]^2$.
\label{prop:least-square}
\end{proposition}

\begin{remark}
    Similar formulations as in Proposition~\ref{prop:least-square} has been previously explored in some literatures proposing stable versions of GANs \citep{mao2017least} and adversarial imitation learning \citep{al2023ls}. The key difference of the formulation proposed in Proposition~\ref{prop:least-square} compared to prior work is that we plugged the closed form solution given by the KL regularized policy optimization objective into the original least square reward learning objective using the reward mapping defined in~\eqref{eqn:reward-mapping}.
\end{remark}

\begin{remark}
    As discussed in Sec.~\ref{sec:single-stage}, leveraging the convex regularizer $\psi(r)=c\alpha\cdot\EE_{\pi^\star}[(r(x,y))^2]+c(1-\alpha)\cdot\EE_{\pi}[(r(x,y))^2]$ with coefficient $\alpha$ for reward learning objective is equivalent to measuring the mixed $\chi^2$ divergence $D_{\chi^2}(\pi^\star\Vert\alpha\pi^\star+(1-\alpha)\pi^{k})$. Therefore, $\alpha$ can be seen as a coefficient for mix-up ratio between oracle data generated from $\pi^\star$ and the data from the previous policy $\pi^{k}$. The data mix-up strategy has been applied in the practical implementation of SPIN \citep{chen2024self}. Usually, we set $\alpha=0.5$ to consider the balanced sampling scenario, i.e., drawing the same amount of data from the oracle policy $\pi^\star$ and the previous step policy $\pi^{k}$.
\end{remark}

By Proposition~\ref{prop:least-square}, we can retrieve the learning objective of our proposed $\chi^2$ self-play imitation finetuning algorithm using finite dataset approximation and balanced sampling ($\alpha=0.5$) for one iteration:
\begin{align}
    \pi^{k+1}{=}\underset{\pi}{\argmin}~\widehat{\mathcal{L}}(\pi){+}\frac{\zeta}{2}\underset{(x,y)\sim\mathcal{D}^\star\cup\mathcal{D}^k}{\EE}\Big[\log\!\frac{\pi(y|x)}{\pi^{k}(y|x)}\Big]^2\!\!,
\label{eqn:chi-squared-self-play}
\end{align}
where $\widehat{\mathcal{L}}(\pi)$ is the empirical loss with dataset $\mathcal{D}^\star$ and $\mathcal{D}^k$:
\begin{align*}
    \widehat{\mathcal{L}}(\pi)&:=\frac{1}{2}~\mathbb{E}_{(x,y)\sim\cD^\star}\Big[\beta\log\frac{\pi(y|x)}{\pi^{k}(y|x)}-r_{\max}\Big]^2\nonumber\\
    &\quad+\frac{1}{2}~\mathbb{E}_{(x,y)\sim\cD^{k}}\Big[\beta\log\frac{\pi(y|x)}{\pi^{k}(y|x)}-r_{\min}\Big]^2\nonumber,
\end{align*}
with $r_{\max}=1/c$ and $r_{\min}=-1/c$. Intuitively, $\widehat{\mathcal{L}}(\pi)$ corresponds to a least squares objective that pushes the rewards of expert samples in $\mathcal{D}^\star$ toward $r_{\max}$ while pulling the rewards of samples from the previous iteration dataset $\mathcal{D}^k$ toward $r_{\min}$, thereby creating a clear margin between the two classes of data for discrimination. The second term in~\eqref{eqn:chi-squared-self-play} serves as a regularizer that mitigates over optimization and encourages the updated reward to remain close to the reward from the previous iteration, which aligns with the mirror descent update applied to the reward player in Algorithm~\ref{alg:self-play-imitation}.
\begin{algorithm}[t]
    \caption{Self-Play Imitation Finetuning (Practical)}
    \begin{algorithmic}[1]
    \label{alg:self-play-imitation-chi}
    \REQUIRE Number of self-play iterations $K$, expert dataset $\mathcal{D}^\star$, reference policy $\pi^{\text{ref}}$, sample size $M$.
    \STATE Initialize $\pi^1=\pi^{\text{ref}}$.
    \FOR{$k = 1, 2, \dots, K$}
        \STATE Sample a dataset $\mathcal{D}^{k}$ using current policy $\pi^{k}$.
        \STATE Update policy via~\eqref{eqn:chi-squared-self-play}.
    \ENDFOR
    \RETURN $\pi^{K+1}$
    \end{algorithmic}
\end{algorithm}

\section{Empirical Results}
Here we present a detailed empirical analysis of our self-play imitation finetuning (SPIF) under the $\chi^2$ divergence, along with a comparative evaluation against SPIN \citep{chen2024self}, SPACE \citep{wang2026space}, and a standard supervised finetuning (SFT) baseline. We conduct experiments using Qwen3-4B and Qwen3-14B~\citep{Yang2025Qwen3TR} and a instruction-following Mistral-7B model \citep{jiang2023mistral7b} on 50k samples subsampled from the UltraChat SFT dataset~\citep{Ding2023EnhancingCL}.

For self-play–based methods (ours, non-linear version of SPIN \citep{chen2024self} and SPACE \citep{wang2026space}), at each iteration $k$ we first generate synthetic responses by sampling $y^{k} \sim \pi^{k}(\cdot \mid x)$ for each prompt $x$ in the SFT dataset. The model is then trained following Algorithm~\ref{alg:self-play-imitation-chi} on data triples $(x, y^\star, y^{k})$.

We evaluate the resulting models on a diverse suite of benchmarks, including Arc Challenge \citep{allenai:arc}, MMLU \citep{Hendrycks2020MeasuringMM}, HellaSwag \citep{Zellers2019HellaSwagCA}, and WinoGrande \citep{Sakaguchi2019WinoGrande}, to assess instruction-following capabilities. Performance results over three self-play iterations are reported in Table~\ref{tab:main-results}. Our method consistently outperforms the supervised finetuning (SFT) baseline, SPIN \citep{chen2024self} and SPACE \citep{wang2026space} across most evaluation settings, demonstrating the effectiveness of the proposed algorithm. We provide detailed implementation descriptions and experimental settings in Appendix~\ref{sec:implementation}, ablation studies in Appendix~\ref{sec:ablation}, additional results on the mathematical capabilities in Appendix~\ref{sec:math} and additional results with six self-play iterations in Appendix~\ref{sec:six-iter}.

\noindent \textbf{Reward Dynamics Analysis.}
We analyze the reward dynamics during training for our SPIF approach with $\chi^2$ divergence, and compare it against the reward behavior of SPIN \citep{chen2024self}. We observe that although our reward exhibits a significantly smaller magnitude, it still effectively discriminates between data sampled from the expert policy $\pi^\star$ and data generated by the previous-iteration model $\pi^k$. Figure~\ref{fig:reward-plot} illustrates the reward distribution at the first self-play iteration for the Qwen-3-4B model. This empirical observation supports the theoretical analysis in Sec~\ref{sec:theoretical-analysis}, which shows that SPIF with $\chi^2$ regularization admits a bounded reward magnitude and consequently enjoys a tighter upper bound on the duality gap, as stated in Theorem~\ref{thm:upper-dual}, with a smaller $R_{\max}$.
\begin{figure}[t]
\centering
\begin{minipage}{0.48\linewidth}
\centering
\includegraphics[width=\linewidth]{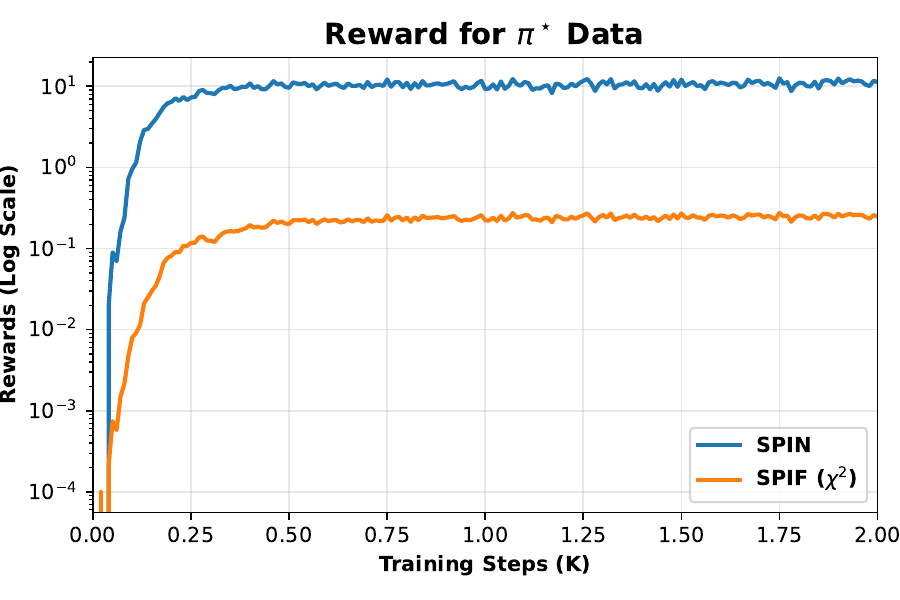}
\end{minipage}\hfill
\begin{minipage}{0.48\linewidth}
\centering
\includegraphics[width=\linewidth]{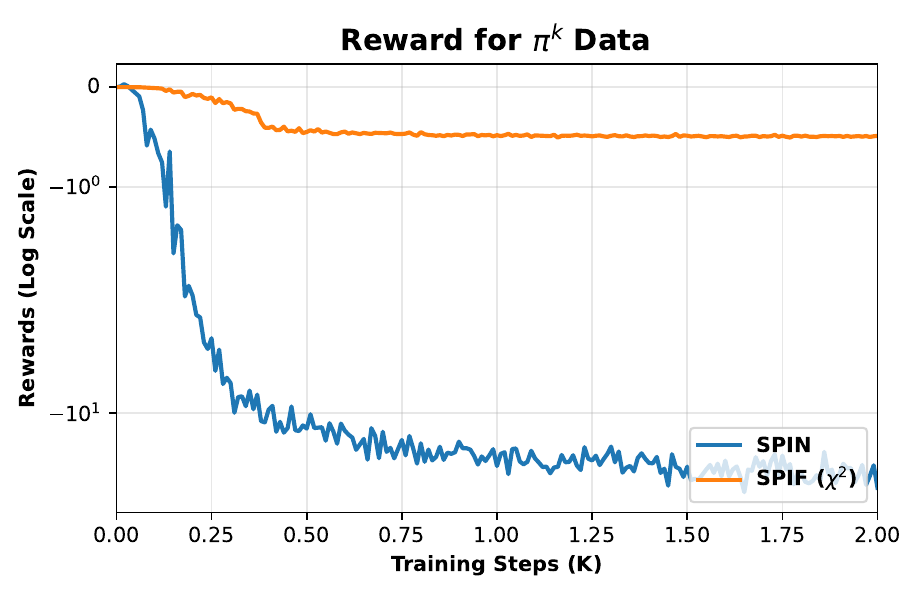}
\end{minipage}
\vspace{-0.5em}
\caption{\textbf{Reward Dynamics Analysis.} We plot the reward curves (log-scaled) during training for our approach, SPIF with $\chi^2$ regularization, and for SPIN \citep{chen2024self}. The results show that our method produces rewards with substantially smaller magnitude, which leads to more stable learning dynamics and is consistent with our theoretical analysis predicting a tighter duality gap.}
\label{fig:reward-plot}
\vspace{-12pt}
\end{figure}

\begin{figure}[t]
\centering
\begin{minipage}{0.48\linewidth}
\centering
\includegraphics[width=\linewidth]{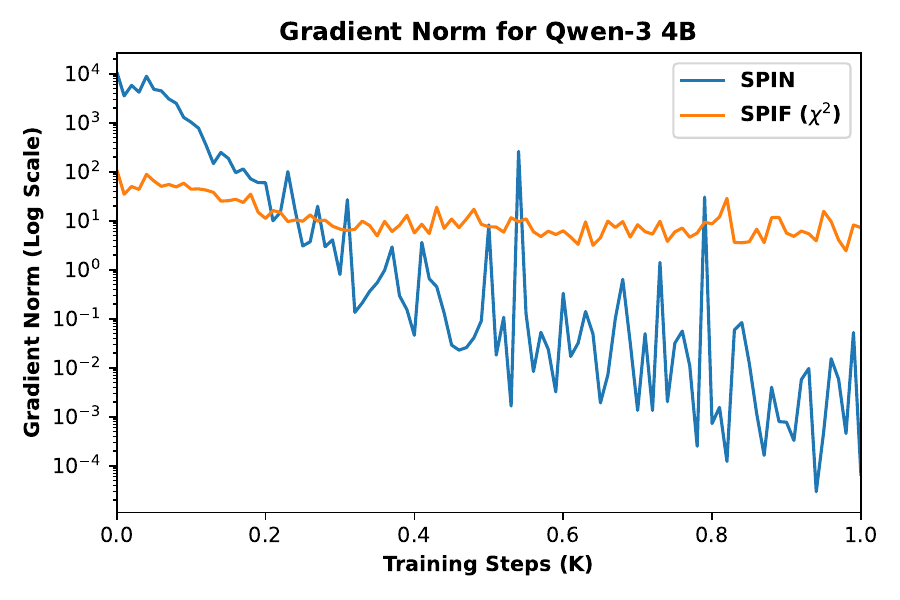}
\end{minipage}\hfill
\begin{minipage}{0.48\linewidth}
\centering
\includegraphics[width=\linewidth]{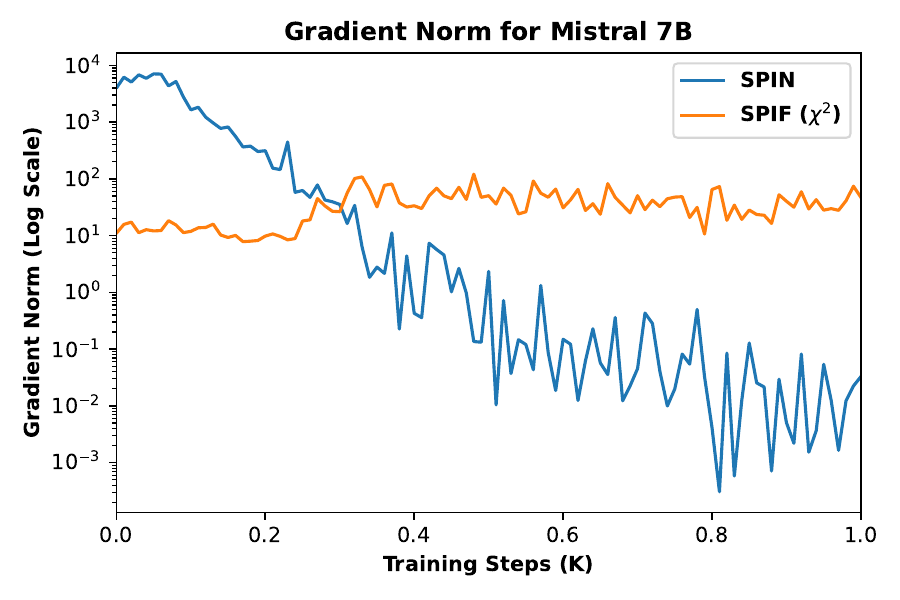}
\end{minipage}
\vspace{-0.5em}
\caption{\textbf{Gradient Norm Analysis.} We plot the gradient norms (log-scaled) during training for our approach, SPIF with $\chi^2$ regularization, and for SPIN \citep{chen2024self}. The results show that our method maintains significantly more stable gradient norms, indicating improved training stability compared to SPIN.}
\label{fig:grad-norm}
\vspace{-15pt}
\end{figure}

\noindent \textbf{Gradient Norm Analysis.}
We report the gradient norm dynamics during training for our SPIF approach with $\chi^2$ regularization and compare them against those of the original SPIN objective \citep{chen2024self}. Under the SPIN objective, the gradient norm is initially very large (near the order of $10^4$) and then rapidly collapses to near zero (on the order of $10^{-4}$), which can lead to unstable optimization behavior. In contrast, our $\chi^2$-regularized approach maintains a relatively small and stable gradient norm, typically in the range of $10^1$ to $10^2$, resulting in more stable training dynamics, as show in in Figure~\ref{fig:grad-norm}.

\section{Conclusion}

We presented a unified theoretical view of self-play post-training for language model alignment by formulating it as adversarial imitation learning.
With a game-theoretic analysis based on this perspective and clarifies the relationship underlying existing methods, we propose a self-play imitation finetuning algorithm based on a $\chi^2$-divergence variational objective, which yields bounded rewards and improved training stability. 
Experiments on various of models demonstrate consistent improvements over prior self-play methods, validating both the theoretical insights and the practical effectiveness of the proposed approach.
\section*{Limitations}

Our theoretical analysis relies on realizability assumptions. While such assumptions are standard in theoretical analyses, they may not always hold in practical LLM finetuning scenarios, where the reward and policy classes can be misspecified. In addition, our convergence result establishes an average-iterate guarantee, whereas practical self-play finetuning evaluates the last-iterate model produced by training. Establishing last-iterate convergence would require stronger assumptions on the game structure or optimization dynamics, which we leave for future work.
\bibliography{references}

@article{chen2024self,
  title={Self-play fine-tuning converts weak language models to strong language models},
  author={Chen, Zixiang and Deng, Yihe and Yuan, Huizhuo and Ji, Kaixuan and Gu, Quanquan},
  journal={arXiv preprint arXiv:2401.01335},
  year={2024}
}

@article{wu2024self,
  title={Self-play preference optimization for language model alignment},
  author={Wu, Yue and Sun, Zhiqing and Yuan, Huizhuo and Ji, Kaixuan and Yang, Yiming and Gu, Quanquan},
  journal={arXiv preprint arXiv:2405.00675},
  year={2024}
}

@article{zhang2025improving,
  title={Improving LLM general preference alignment via optimistic online mirror descent},
  author={Zhang, Yuheng and Yu, Dian and Ge, Tao and Song, Linfeng and Zeng, Zhichen and Mi, Haitao and Jiang, Nan and Yu, Dong},
  journal={arXiv preprint arXiv:2502.16852},
  year={2025}
}

@article{liu2021provably,
  title={Provably efficient generative adversarial imitation learning for online and offline setting with linear function approximation},
  author={Liu, Zhihan and Zhang, Yufeng and Fu, Zuyue and Yang, Zhuoran and Wang, Zhaoran},
  journal={arXiv preprint arXiv:2108.08765},
  year={2021}
}

@inproceedings{cai2020provably,
  title={Provably efficient exploration in policy optimization},
  author={Cai, Qi and Yang, Zhuoran and Jin, Chi and Wang, Zhaoran},
  booktitle={International Conference on Machine Learning},
  pages={1283--1294},
  year={2020},
  organization={PMLR}
}

@article{garg2021iq,
  title={Iq-learn: Inverse soft-q learning for imitation},
  author={Garg, Divyansh and Chakraborty, Shuvam and Cundy, Chris and Song, Jiaming and Ermon, Stefano},
  journal={Advances in Neural Information Processing Systems},
  volume={34},
  pages={4028--4039},
  year={2021}
}

@article{foster2024behavior,
  title={Is behavior cloning all you need? understanding horizon in imitation learning},
  author={Foster, Dylan J and Block, Adam and Misra, Dipendra},
  journal={Advances in Neural Information Processing Systems},
  volume={37},
  pages={120602--120666},
  year={2024}
}

@article{huang2024correcting,
  title={Correcting the mythos of kl-regularization: Direct alignment without overoptimization via chi-squared preference optimization},
  author={Huang, Audrey and Zhan, Wenhao and Xie, Tengyang and Lee, Jason D and Sun, Wen and Krishnamurthy, Akshay and Foster, Dylan J},
  journal={arXiv preprint arXiv:2407.13399},
  year={2024}
}

@inproceedings{mao2017least,
  title={Least squares generative adversarial networks},
  author={Mao, Xudong and Li, Qing and Xie, Haoran and Lau, Raymond YK and Wang, Zhen and Paul Smolley, Stephen},
  booktitle={Proceedings of the IEEE international conference on computer vision},
  pages={2794--2802},
  year={2017}
}

@article{al2023ls,
  title={Ls-iq: Implicit reward regularization for inverse reinforcement learning},
  author={Al-Hafez, Firas and Tateo, Davide and Arenz, Oleg and Zhao, Guoping and Peters, Jan},
  journal={arXiv preprint arXiv:2303.00599},
  year={2023}
}

@article{zhang2024iterative,
  title={Iterative nash policy optimization: Aligning llms with general preferences via no-regret learning},
  author={Zhang, Yuheng and Yu, Dian and Peng, Baolin and Song, Linfeng and Tian, Ye and Huo, Mingyue and Jiang, Nan and Mi, Haitao and Yu, Dong},
  journal={arXiv preprint arXiv:2407.00617},
  year={2024}
}

@article{Yang2025Qwen3TR,
  title={Qwen3 Technical Report},
  author={An Yang and Anfeng Li and Baosong Yang and Beichen Zhang and Binyuan Hui and Bo Zheng and Bowen Yu and Chang Gao and Chengen Huang and Chenxu Lv and Chujie Zheng and Dayiheng Liu and Fan Zhou and Fei Huang and Feng Hu and Hao Ge and Haoran Wei and Huan Lin and Jialong Tang and Jian Yang and Jianhong Tu and Jianwei Zhang and Jianxin Yang and Jiaxin Yang and Jingren Zhou and Jingren Zhou and Junyan Lin and Kai Dang and Keqin Bao and Ke‐Pei Yang and Le Yu and Li-Chun Deng and Mei Li and Min Xue and Mingze Li and Pei Zhang and Peng Wang and Qin Zhu and Rui Men and Ruize Gao and Shi-Qiang Liu and Shuang Luo and Tianhao Li and Tianyi Tang and Wenbiao Yin and Xingzhang Ren and Xinyu Wang and Xinyu Zhang and Xuancheng Ren and Yang Fan and Yang Su and Yi-Chao Zhang and Yinger Zhang and Yu Wan and Yuqiong Liu and Zekun Wang and Zeyu Cui and Zhenru Zhang and Zhipeng Zhou and Zihan Qiu},
  journal={ArXiv},
  year={2025},
  volume={abs/2505.09388}
}

@article{Ding2023EnhancingCL,
  title={Enhancing Chat Language Models by Scaling High-quality Instructional Conversations},
  author={Ning Ding and Yulin Chen and Bokai Xu and Yujia Qin and Zhi Zheng and Shengding Hu and Zhiyuan Liu and Maosong Sun and Bowen Zhou},
  journal={ArXiv},
  year={2023},
  volume={abs/2305.14233},
}

@article{allenai:arc,
      author    = {Peter Clark  and Isaac Cowhey and Oren Etzioni and Tushar Khot and
                    Ashish Sabharwal and Carissa Schoenick and Oyvind Tafjord},
      title     = {Think you have Solved Question Answering? Try ARC, the AI2 Reasoning Challenge},
      journal   = {arXiv:1803.05457v1},
      year      = {2018},
}

@article{Hendrycks2020MeasuringMM,
  title={Measuring Massive Multitask Language Understanding},
  author={Dan Hendrycks and Collin Burns and Steven Basart and Andy Zou and Mantas Mazeika and Dawn Xiaodong Song and Jacob Steinhardt},
  journal={ArXiv},
  year={2020},
  volume={abs/2009.03300}
}

@inproceedings{Zellers2019HellaSwagCA,
  title={HellaSwag: Can a Machine Really Finish Your Sentence?},
  author={Rowan Zellers and Ari Holtzman and Yonatan Bisk and Ali Farhadi and Yejin Choi},
  booktitle={Annual Meeting of the Association for Computational Linguistics},
  year={2019}
}

@article{Sakaguchi2019WinoGrande,
  title={WinoGrande},
  author={Keisuke Sakaguchi and Ronan Le Bras and Chandra Bhagavatula and Yejin Choi},
  journal={Communications of the ACM},
  year={2019},
  volume={64},
  pages={99 - 106}
}

@misc{jiang2023mistral7b,
      title={Mistral 7B}, 
      author={Albert Q. Jiang and Alexandre Sablayrolles and Arthur Mensch and Chris Bamford and Devendra Singh Chaplot and Diego de las Casas and Florian Bressand and Gianna Lengyel and Guillaume Lample and Lucile Saulnier and Lélio Renard Lavaud and Marie-Anne Lachaux and Pierre Stock and Teven Le Scao and Thibaut Lavril and Thomas Wang and Timothée Lacroix and William El Sayed},
      year={2023},
      eprint={2310.06825},
      archivePrefix={arXiv},
      primaryClass={cs.CL},
}

@article{chi2024diffusionpolicy,
	author = {Cheng Chi and Zhenjia Xu and Siyuan Feng and Eric Cousineau and Yilun Du and Benjamin Burchfiel and Russ Tedrake and Shuran Song},
	title ={Diffusion Policy: Visuomotor Policy Learning via Action Diffusion},
	journal = {The International Journal of Robotics Research},
	year = {2024},
}

@inproceedings{florence2022implicit,
  title={Implicit behavioral cloning},
  author={Florence, Pete and Lynch, Corey and Zeng, Andy and Ramirez, Oscar A and Wahid, Ayzaan and Downs, Laura and Wong, Adrian and Lee, Johnny and Mordatch, Igor and Tompson, Jonathan},
  booktitle={Conference on robot learning},
  pages={158--168},
  year={2022},
  organization={PMLR}
}

@article{rohatgi2025computational,
  title={Computational-statistical tradeoffs at the next-token prediction barrier: Autoregressive and imitation learning under misspecification},
  author={Rohatgi, Dhruv and Block, Adam and Huang, Audrey and Krishnamurthy, Akshay and Foster, Dylan J},
  journal={arXiv preprint arXiv:2502.12465},
  year={2025}
}

@article{ho2016generative,
  title={Generative adversarial imitation learning},
  author={Ho, Jonathan and Ermon, Stefano},
  journal={Advances in neural information processing systems},
  volume={29},
  year={2016}
}

@article{xu2024provably,
  title={Provably and practically efficient adversarial imitation learning with general function approximation},
  author={Xu, Tian and Zhang, Zhilong and Chen, Ruishuo and Sun, Yihao and Yu, Yang},
  journal={Advances in Neural Information Processing Systems},
  volume={37},
  pages={66108--66146},
  year={2024}
}

@article{li2025near,
  title={Near-Optimal Second-Order Guarantees for Model-Based Adversarial Imitation Learning},
  author={Li, Shangzhe and Zhou, Dongruo and Zhang, Weitong},
  journal={arXiv preprint arXiv:2510.09487},
  year={2025}
}

@article{calandriello2024human,
  title={Human alignment of large language models through online preference optimisation},
  author={Calandriello, Daniele and Guo, Daniel and Munos, Remi and Rowland, Mark and Tang, Yunhao and Pires, Bernardo Avila and Richemond, Pierre Harvey and Lan, Charline Le and Valko, Michal and Liu, Tianqi and others},
  journal={arXiv preprint arXiv:2403.08635},
  year={2024}
}

@inproceedings{mishra2022cross,
  title={Cross-task generalization via natural language crowdsourcing instructions},
  author={Mishra, Swaroop and Khashabi, Daniel and Baral, Chitta and Hajishirzi, Hannaneh},
  booktitle={Proceedings of the 60th Annual Meeting of the Association for Computational Linguistics (Volume 1: Long Papers)},
  pages={3470--3487},
  year={2022}
}

@article{thoppilan2022lamda,
  title={Lamda: Language models for dialog applications},
  author={Thoppilan, Romal and De Freitas, Daniel and Hall, Jamie and Shazeer, Noam and Kulshreshtha, Apoorv and Cheng, Heng-Tze and Jin, Alicia and Bos, Taylor and Baker, Leslie and Du, Yu and others},
  journal={arXiv preprint arXiv:2201.08239},
  year={2022}
}

@article{chung2024scaling,
  title={Scaling instruction-finetuned language models},
  author={Chung, Hyung Won and Hou, Le and Longpre, Shayne and Zoph, Barret and Tay, Yi and Fedus, William and Li, Yunxuan and Wang, Xuezhi and Dehghani, Mostafa and Brahma, Siddhartha and others},
  journal={Journal of Machine Learning Research},
  volume={25},
  number={70},
  pages={1--53},
  year={2024}
}

@article{ouyang2022training,
  title={Training language models to follow instructions with human feedback},
  author={Ouyang, Long and Wu, Jeffrey and Jiang, Xu and Almeida, Diogo and Wainwright, Carroll and Mishkin, Pamela and Zhang, Chong and Agarwal, Sandhini and Slama, Katarina and Ray, Alex and others},
  journal={Advances in neural information processing systems},
  volume={35},
  pages={27730--27744},
  year={2022}
}

@article{bai2022training,
  title={Training a helpful and harmless assistant with reinforcement learning from human feedback},
  author={Bai, Yuntao and Jones, Andy and Ndousse, Kamal and Askell, Amanda and Chen, Anna and DasSarma, Nova and Drain, Dawn and Fort, Stanislav and Ganguli, Deep and Henighan, Tom and others},
  journal={arXiv preprint arXiv:2204.05862},
  year={2022}
}

@article{rafailov2023direct,
  title={Direct preference optimization: Your language model is secretly a reward model},
  author={Rafailov, Rafael and Sharma, Archit and Mitchell, Eric and Manning, Christopher D and Ermon, Stefano and Finn, Chelsea},
  journal={Advances in neural information processing systems},
  volume={36},
  pages={53728--53741},
  year={2023}
}

@article{guo2025deepseek,
  title={Deepseek-r1: Incentivizing reasoning capability in llms via reinforcement learning},
  author={Guo, Daya and Yang, Dejian and Zhang, Haowei and Song, Junxiao and Zhang, Ruoyu and Xu, Runxin and Zhu, Qihao and Ma, Shirong and Wang, Peiyi and Bi, Xiao and others},
  journal={arXiv preprint arXiv:2501.12948},
  year={2025}
}

@inproceedings{abbeel2004apprenticeship,
  title={Apprenticeship learning via inverse reinforcement learning},
  author={Abbeel, Pieter and Ng, Andrew Y},
  booktitle={Proceedings of the twenty-first international conference on Machine learning},
  pages={1},
  year={2004}
}

@article{ren2024hybrid,
  title={Hybrid inverse reinforcement learning},
  author={Ren, Juntao and Swamy, Gokul and Wu, Zhiwei Steven and Bagnell, J Andrew and Choudhury, Sanjiban},
  journal={arXiv preprint arXiv:2402.08848},
  year={2024}
}

@article{rafailov2021visual,
  title={Visual adversarial imitation learning using variational models},
  author={Rafailov, Rafael and Yu, Tianhe and Rajeswaran, Aravind and Finn, Chelsea},
  journal={Advances in Neural Information Processing Systems},
  volume={34},
  pages={3016--3028},
  year={2021}
}

@article{ablett2023learning,
  title={Learning from guided play: Improving exploration for adversarial imitation learning with simple auxiliary tasks},
  author={Ablett, Trevor and Chan, Bryan and Kelly, Jonathan},
  journal={IEEE Robotics and Automation Letters},
  volume={8},
  number={3},
  pages={1263--1270},
  year={2023},
  publisher={IEEE}
}

@article{hejna2023contrastive,
  title={Contrastive preference learning: learning from human feedback without rl},
  author={Hejna, Joey and Rafailov, Rafael and Sikchi, Harshit and Finn, Chelsea and Niekum, Scott and Knox, W Bradley and Sadigh, Dorsa},
  journal={arXiv preprint arXiv:2310.13639},
  year={2023}
}

@article{tu2025enhancing,
  title={Enhancing LLM Reasoning with Iterative DPO: A Comprehensive Empirical Investigation},
  author={Tu, Songjun and Lin, Jiahao and Tian, Xiangyu and Zhang, Qichao and Li, Linjing and Fu, Yuqian and Xu, Nan and He, Wei and Lan, Xiangyuan and Jiang, Dongmei and others},
  journal={arXiv preprint arXiv:2503.12854},
  year={2025}
}

@article{cobbe2021gsm8k,
  title={Training Verifiers to Solve Math Word Problems},
  author={Cobbe, Karl and Kosaraju, Vineet and Bavarian, Mohammad and Chen, Mark and Jun, Heewoo and Kaiser, Lukasz and Plappert, Matthias and Tworek, Jerry and Hilton, Jacob and Nakano, Reiichiro and Hesse, Christopher and Schulman, John},
  journal={arXiv preprint arXiv:2110.14168},
  year={2021}
}

@article{wang2026space,
  title={SPACE: Noise contrastive estimation stabilizes self-play fine-tuning for large language models},
  author={Wang, Yibo and Huzhang, Guangda and Chen, Qingguo and Xu, Zhao and Luo, Weihua and Zhang, Kaifu and Zhang, Lijun},
  journal={Advances in Neural Information Processing Systems},
  volume={38},
  pages={60042--60072},
  year={2026}
}

@article{wang2026triplets,
  title={Triplets better than pairs: Towards stable and effective self-play fine-tuning for LLMs},
  author={Wang, Yibo and Sun, Hai-Long and Huzhang, Guangda and Chen, Qingguo and Xu, Zhao and Luo, Weihua and Zhang, Kaifu and Zhang, Lijun},
  journal={Advances in Neural Information Processing Systems},
  volume={38},
  pages={40980--41009},
  year={2026}
}

\newpage
\appendix
\onecolumn

\section{Discussion}
\label{sec:discussion}
\subsection{Adversarial Imitation Learning.}
Existing adversarial imitation learning methods typically optimize a single-stage min–max objective as in~\eqref{eqn:ail-formulation}, but are formulated under the full MDP setting rather than the contextual bandit setting considered in this paper. In particular, GAIL \citep{ho2016generative} employs a regularization that makes the resulting objective equivalent to minimizing the Jensen–Shannon divergence between expert and behavioral distributions. IQ-Learn \citep{garg2021iq} uses the regularizer $\psi(r)=\mathbb{E}_{\rho^\star}[r^2]$, which is equivalent to minimizing the $\chi^2$ divergence between the expert occupancy measure $\rho^\star$ and the learner occupancy measure $\rho^\pi$, where $\rho$ denotes the policy occupancy measure. LS-IQ \citep{al2023ls} further considers the regularizer $\psi(r)=\alpha~\mathbb{E}_{\rho^\star}[r^2] + (1-\alpha)~\mathbb{E}_{\rho^\pi}[r^2]$, which corresponds to minimizing the $\chi^2$ divergence between $\rho^\star$ and a mixture of $\rho^\star$ and $\rho^\pi$ with mixing coefficient $\alpha$. Notably, the regularization used in LS-IQ is closely related to the formulation adopted in this paper.

\subsection{Self-Play Imitation Finetuning}
Self-play imitation finetuning refers to methods that leverage an SFT dataset to perform self-play, with the goal of imitating behaviors in the SFT data rather than optimizing external signals such as preferences. Both SPIN \citep{chen2024self} and our proposed Algorithm~\ref{alg:self-play-imitation-chi} fall into this category. We show that both methods can be formulated within the standard adversarial imitation learning framework in~\eqref{eqn:ail-formulation}.

In particular, the linear variant of SPIN (using idential link function $\sigma(t)=t$) directly minimizes the total variation distance between the expert policy $\pi^\star$ and the learned policy $\pi$, as proved in Appendix~\ref{sec:linear-spin-tv}. The nonlinear variant of SPIN ($\sigma(t)=-\log(1+\exp(-t))$) can be viewed as minimizing a KL divergence between $\pi^\star$ and $\pi$, as shown in Appendix~\ref{sec:non-linear-spin}.

As self-play methods can be cast as imitation learning toward a prescribed objective (Table~\ref{tab:divergences}), they implicitly induce a capacity ceiling determined by the underlying imitation target. Consequently, the model capability attainable through such procedures is fundamentally bounded, implying that existing LLM self-play algorithms cannot achieve infinite capability gains via iterative self-improvement alone.

\subsection{Self-Play for General Preference Alignment}
In this sections, we extend our discussion to the general preference alignment with self-play without the need of Bradley-Terry preference model~\citep{wu2024self,zhang2024iterative,zhang2025improving}.
Similar with our proposed $\chi^2$ self-play imitation finetuning, these methods often rely on the squared loss in different settings. 
Below, we show that this line of work admits an AIL interpretation with respect to a general preference oracle and 
exhibit a close connection to $\chi^2$ divergence–regularized adversarial imitation learning. 
\looseness=-1

\noindent \textbf{SPPO.} Given an estimated probability $P(y\!\succ\! \pi^{k}|x)$ using the general preference model, the objective in SPPO~\citep{wu2024self} for policy $\pi^{k+1}$ can be written as:
\begin{align*}
\underset{\pi}{\argmin}\underset{\rho,\pi^k}{\EE}\Big[\!
\log\!\frac{\pi(y|x)}{\pi^{k}(y|x)} {-} \frac{1}{\beta}\!\left(P(y\succ \pi^{k}|x) - \tfrac{1}{2}\right)
\!\!\Big]^2.
\end{align*}
This update rule shares similarities with our proposed objective in~\eqref{eqn:chi-squared-self-play}
by optimizing the $\chi^2$ regularized AIL objective under a trust region constraint by slightly generalizing our proposed imitation framework in Sec.~\ref{sec:ail-view} to preference-based policy optimization using the following proposition:
\begin{proposition}
Given a preference model that outputs $w^k(x,y):=P(y\succ \pi^k\mid x)$, let
$y\sim\pi^k(\cdot|x)$ and define the weighted expectations:
\begin{align*}
\EE_{\pi_+^k}[f(x,y)] &:= \EE_{y\sim\pi^k(\cdot|x)}\!\left[w^k(x,y)\,f(x,y)\right],\\
\EE_{\pi_-^k}[f(x,y)] &:= \EE_{y\sim\pi^k(\cdot|x)}\!\left[(1-w^k(x,y))\,f(x,y)\right].
\end{align*}
With $\sigma(t) = t$, mixed $\chi^2$ regularizer and $r(x, y)$ being the reparameterized reward function in~\eqref{eqn:reward-mapping}, the SPPO objective is equivalent to optimizing:
\begin{align*}
\arg\max_{r}\;
&\EE_{\rho}\Big[
\sigma\!\big(
\EE_{\pi_+^{k}}[r(x,y)]
-
\EE_{\pi_-^{k}}[r(x,y)]
\big)
-
\phi(r,r^{k-1})
\Big],
\end{align*}
\vspace{-12pt}
\label{prop:sppo}
\end{proposition}
\begin{remark}
Proposition~\ref{prop:sppo} provides a new derivation of SPPO algorithm, differs from the original derivation from \citet{wu2024self}. This new formulation
    shows that instead of imitating the expert policy $\pi^\star$ in original AIL formulation shown in~\eqref{eqn:two-stage-formulation}, SPPO is adversarially imitating the general preference oracle $P(y\succ \pi^k\mid x)$.
\end{remark}

\noindent \textbf{Iterative DPO.} Iterative DPO serves as a baseline in SPPO \citep{wu2024self} and has also been applied to enhance LLM reasoning \citep{tu2025enhancing}. In practice, it replaces the SPPO loss with a DPO-style objective while keeping the remaining components unchanged. In Appendix~\ref{sec:iterative_dpo}, we show that this procedure implicitly minimizes the KL divergence between the model policy and an expert policy $\pi^\star$ induced by the preference oracle iteratively.

\noindent \textbf{INPO.} INPO \citep{zhang2024iterative} formulates the general preference alignment problem as solving Nash policy with online mirror descent. For each iteration, it updates through the following update rule:
\begin{align*}
\pi^{k+1}=\argmin_\pi\underset{\substack{y, y' \sim \pi^k(\cdot | x) \\ x \sim \rho(x) \\ y_w, y_l \sim \lambda_p(y, y')}}{\EE}\bigg[\log\frac{\pi(y_w|x)}{\pi(y_l|x)} \notag -\frac{\tau}{\eta}\log\frac{\pi_{\text{ref}}(y_w|x)}{\pi_{\text{ref}}(y_l|x)} -\frac{\eta-\tau}{\eta}\log\frac{\pi^k(y_w|x)}{\pi^k(y_l|x)}-\frac{1}{2\eta}\bigg]^2.
\end{align*}
We show that this update rule is equivalent to a iterative AIL procedure operating on the preference oracle. Compared to SPPO \citep{wu2024self}, this equivalence arises under a different construction of $\pi_+^k$ and $\pi_-^k$, as well as a distinct reparameterization of $r$.
\begin{proposition}
Given a preference model that outputs $w^k(x,y):=P(y\succ \pi^k\mid x)$, let
$(y,y')\sim \pi^k(\cdot|x)\times \pi^k(\cdot|x)$ and define the weighted pairwise expectations:
\begin{align*}
\EE_{\pi_+^k}[f(x,y,y')] &:= \EE_{(y,y')\sim \pi^k\times\pi^k}\!\left[w^k(x,y)\, f(x,y,y')\right],\\
\EE_{\pi_-^k}[f(x,y,y')] &:= \EE_{(y,y')\sim \pi^k\times\pi^k}\!\left[w^k(x,y')\, f(x,y,y')\right].
\end{align*}
Then the INPO objective is equivalent to optimizing the reparameterized $r$ with the mixed $\chi^2$ regularizer for:
\begin{align*}
\arg\max_{r}\; \EE_{x\sim\rho}\Big[
\sigma\!\big(\EE_{\pi_+^k}[r(x,y,y')]-\EE_{\pi_-^k}[r(x,y,y')]\big)
-\phi(r,r^{k-1})
\Big],
\end{align*}
where $\sigma(t)=t$, and $r(x,y,y')$ is defined in:
\begin{align} r(x,y,y'):=\eta\!\log\frac{\pi(y|x)}{\pi(y'|x)}{-}{\tau}\!\log\frac{\pi_{\mathrm{ref}}(y|x)}{\pi_{\mathrm{ref}}(y'|x)}{-}{(\eta{-}\tau)}\!\log\frac{\pi^k(y|x)}{\pi^k(y'|x)}. \label{eqn:delta-r-inpo} \end{align}
\label{prop:inpo}
\vspace{-5pt}
\end{proposition}
\begin{remark}
The reward reparameterization defined in~\eqref{eqn:delta-r-inpo} does not directly arise in closed form from a direct mirror descent optimization as in~\eqref{eqn:two-stage-formulation}. Instead, it has an additional KL regularizer $D_{\text{KL}}(\pi\Vert\pi_{\text{ref}})$, reflecting the fact that INPO \citep{zhang2024iterative} formulates a constrained optimization problem with respect to the reference policy $\pi_{\text{ref}}$. The resulting formulation involves a paired response $(y, y')$ for canceling the partition function. Nevertheless, the overall AIL interpretation remains consistent.
\end{remark}

\noindent \textbf{ONPO.} ONPO \citep{zhang2025improving} extends INPO by replacing standard online mirror descent with optimistic online mirror descent. Under the assumption that $m_k=\mathbb{E}_{y'\sim\pi_{k-1}(\cdot|x)}[P(y\succ y')]$ is known, and by introducing an additional policy player, ONPO achieves an improved duality gap upper bound of $\mathcal{O}(1/K)$, in contrast to the standard $\mathcal{O}(1/\sqrt{K})$ rate of online mirror descent. By adopting the same assumption and augmenting Algorithm~\ref{alg:self-play-imitation} with an additional policy player, our framework can similarly strengthen the result in Theorem~\ref{thm:upper-dual} to an $\mathcal{O}(1/K)$ rate.

\section{Proof of Theorem~\ref{thm:upper-dual}}
\subsection{Key Lemmas}
We first introduce the following lemmas:

\begin{lemma}[One-Step Descent, \citealt{cai2020provably}]
    For two policy distributions $\pi^\star$ and $\pi$, and a reward function $r:\mathcal{X}\times\mathcal{Y}\rightarrow[-R_{\max},R_{\max}]$, it holds for $\pi'(\cdot|x)\propto\pi(\cdot|x)\cdot\exp(\eta\cdot r(x,\cdot))$ that:
    \begin{align*}
       \langle r(x, \cdot),\pi^\star(\cdot|x)-\pi(\cdot|x)\rangle \leq \frac{\eta R^2_{\max}}{2}+\eta^{-1}\cdot\left(D_{\text{KL}}(\pi^\star(\cdot|x)\Vert\pi(\cdot|x))-D_{\text{KL}}(\pi^\star(\cdot|x)\Vert\pi'(\cdot|x))\right)
    \end{align*}
\label{lem:one-step}
\end{lemma}
\begin{proof}
For any function $r:\cX\times\cY\rightarrow\RR$ and distributions $\pi(\cdot|x),\pi'(\cdot|x) \in\Delta(\cY)$ that satisfy 
\begin{align*}
\pi'(\cdot|x)\propto \pi(\cdot|x)\cdot \exp\bigl(\eta\cdot r(x, \cdot)\bigr),
\end{align*}
we have 
\begin{align}
\eta\cdot \la r, \pi^\star - \pi' \ra 
&=  \la Z+\log(\pi'/\pi), \pi^\star-\pi' \ra \nonumber\\
&=\la Z, \pi^\star-\pi' \ra + \la \log(\pi^\star/\pi), \pi^\star \ra + \la \log(\pi'/\pi^\star), \pi^\star \ra  + \la \log(\pi'/\pi), -\pi' \ra\nonumber \\
&= D_{\text{KL}}(\pi^\star\,\|\,\pi) - D_{\text{KL}}(\pi^\star\,\|\,\pi') - D_{\text{KL}}(\pi'\,\|\,\pi).
\label{eqn:kl-decompose}
\end{align}
Here $Z: \cX \rightarrow \RR$ is a constant function defined by
\begin{align*}
Z(x) = \log\Bigl(\sum_{y\in\cY}\pi(y|x)\cdot\exp\bigl(\eta\cdot r(x,y)\bigr)\Bigr),
\end{align*}
which implies that $\la Z, \pi^\star-\pi' \ra = 0$ in~\eqref{eqn:kl-decompose} as $\pi'(\cdot|x),\pi^\star(\cdot|x) \in\Delta(\cY)$. Moreover, by~\eqref{eqn:kl-decompose} we have
\begin{align}
\eta \cdot \la r(x,\cdot), \pi^\star(\cdot) - \pi(\cdot|x) \ra
&= \eta \cdot \la r(x,\cdot), \pi^\star(\cdot) - \pi'(\cdot|x) \ra - \eta \cdot \la r(x,\cdot), \pi(\cdot|x) - \pi'(\cdot|x) \ra \nonumber\\
&  \le D_{\text{KL}}\bigl(\pi^\star(\cdot|x)\,\big\|\,\pi(\cdot|x)\bigr) - D_{\text{KL}}\bigl(\pi^\star(\cdot|x)\,\big\|\,\pi'(\cdot|x)\bigr) - D_{\text{KL}}\bigl(\pi'(\cdot|x)\,\big\|\,\pi(\cdot|x)\bigr) \nonumber\\
&\qquad+ \eta \cdot \|r(x,\cdot)\|_{\infty}\cdot \| \pi(\cdot|x) - \pi'(\cdot|x)\|_1
\label{eqn:1-norm}
\end{align}
for any state $x \in \cX$. Meanwhile, by Pinsker's inequality, it holds that 
\begin{align}
D_{\text{KL}}(\pi' \,\|\, {\pi})\ge \|{\pi}-\pi'\|^2_1/2.
\label{eqn:pinsker}
\end{align}
Combining~\eqref{eqn:1-norm},~\eqref{eqn:pinsker}, and the fact that $\|r(x,\cdot)\|_{\infty}\le R_{\max}$, with Lemma~\ref{lem:kl-upper}, for any state $x \in \cX$, we obtain
\begin{align*}
\eta \cdot \la r(x,\cdot), \pi^\star(\cdot) - \pi(\cdot|x) \ra 
& \le D_{\text{KL}}\bigl(\pi^\star(\cdot|x)\,\big\|\,\pi(\cdot|x)\bigr) - D_{\text{KL}}\bigl(\pi^\star(\cdot|x)\,\big\|\,\pi'(\cdot|x)\bigr) - \| \pi(\cdot|x) - \pi'(\cdot|x)\|_1^2/2 \\
&\qquad + \eta R_{\max} \cdot \| \pi(\cdot|x) - \pi'(\cdot|x)\|_1 \\
& \le D_{\text{KL}}\bigl(\pi^\star(\cdot|x)\,\big\|\,\pi(\cdot|x)\bigr) - D_{\text{KL}}\bigl(\pi^\star(\cdot|x)\,\big\|\,\pi'(\cdot|x)\bigr)  + R^2_{\max}\eta^2/2,
\end{align*}
which concludes the proof.
\end{proof}

\begin{lemma}
If the $r$-player optimizes $r^k$ using Online Mirror Descent (OMD) against a $\pi$-player updated via Line 5 of Algorithm~\ref{alg:self-play-imitation}, and letting $BR^2_{\max}=\max_{r\in\mathcal{R}}D_{f}(r^\star,r)$, the regret of the $r$-player is bounded by:
\begin{align*}
    \max_{r\in\mathcal{R}} \sum_{k=1}^K \langle r, \pi^\star-\pi^k \rangle - \sum_{k=1}^K \langle r^k, \pi^\star-\pi^k \rangle \leq \mathcal{O}(R^2_{\max}B\sqrt{K})
\end{align*}
\label{lem:omd}
\end{lemma}

\begin{proof}
The $r$-player runs online mirror descent to maximize the sequence of objectives $\langle r, \pi^\star-\pi^k\rangle$. The gradients are $g^k \coloneqq \nabla_r \langle r,\pi^\star-\pi^k\rangle = \pi^\star-\pi^k$. For a linear objective and bounded Bregman divergence $\max_{r\in\mathcal{R}} D_f(r^\star,r)= R^2_{\max}B$, selecting the step size $\zeta=\sqrt{K}/(BR^2_{\max})$ yields the standard OMD regret bound for a maximizer:
$$
\max_{r\in\mathcal{R}} \sum_{k=1}^K \langle r, \pi^\star-\pi^k\rangle - \sum_{k=1}^K \langle r^k, \pi^\star-\pi^k\rangle \le \mathcal{O}(R^2_{\max}B\sqrt{K}),
$$
which concludes the proof.
\end{proof}

\subsection{Supporting Lemmas}

\begin{lemma}
\label{lem:kl-upper}
Let $\pi(\cdot|x) \in \Delta(\mathcal{Y})$ be a probability distribution over a discrete action set $\mathcal{Y}$ conditioned on state $x$, and let $r:\mathcal{X}\times\mathcal{Y}\to [-R_{\max},R_{\max}]$. Define the updated policy
\[
\pi'(y|x) \propto \pi(y|x) \exp(\eta r(x,y)), \quad \forall y \in \mathcal{Y}.
\]
Then the KL divergence between $\pi'$ and $\pi$ satisfies
\[
D_{\text{KL}}(\pi'(\cdot|x) \| \pi(\cdot|x)) \le \frac{1}{2} \eta^2R^2_{\max}.
\]
\end{lemma}

\begin{proof}
Fix $x$ and omit its notation for simplicity. 
Let
\[
Z=\mathbb{E}_{\pi}\!\left[e^{\eta r}\right]
=\sum_{y\in\cY} \pi(y|x)e^{\eta r(x,y)}, 
\quad f(\eta)=\log Z.
\]
The updated distribution can be written as
\[
\pi'(y|x)=\frac{\pi(y|x)e^{\eta r(x,y)}}{Z}.
\]
The KL divergence can be expressed as
\[
D_{\text{KL}}(\pi'\|\pi)
=\sum_{y\in\cY} \pi'(y|x)\log\frac{\pi'(y|x)}{\pi(y|x)}
=\eta\,\mathbb{E}_{\pi'}[r]-\log Z.
\]
Let $\mu=\mathbb{E}_{\pi}[r]$, and define the centered cumulant generating function
\[
\psi(\eta)=f(\eta)-\eta\mu
=\log\mathbb{E}_{\pi}\!\left[e^{\eta (r-\mu)}\right].
\]
It follows that
\[
\psi(0)=0, \quad 
\psi'(\eta)=\mathbb{E}_{\pi'}[r]-\mu, \quad 
\psi''(\eta)=\mathrm{Var}_{\pi'}(r).
\]
Thus,
\[
D_{\text{KL}}(\pi'\|\pi)
=\eta\psi'(\eta)-\psi(\eta)
=\int_0^\eta t\,\psi''(t)\,dt.
\]
Since $\|r\|_\infty\le R_{\max}$, we have $\psi''(t)=\mathrm{Var}_{\pi_t}(r)\le R_{\max}^2$ for all $t$. 
Therefore,
\[
D_{\text{KL}}(\pi'\|\pi)
\le \int_0^\eta t R_{\max}^2\,dt
= \frac{1}{2}\eta^2 R_{\max}^2.
\]
\end{proof}

\subsection{Detailed Proof}
\begin{proof}[Proof of Theorem~\ref{thm:upper-dual}]
According to Definition~\ref{def:dual-gap}, the duality gap is defined as:
\begin{align*}
    \mathrm{DualGap} &= \max_{r\in\mathcal{R}}{J}(\bar\pi,r) - \min_{\pi\in\Pi}{J}(\pi,\bar r) \\
    &= \max_{r\in\mathcal{R}} \langle r, \pi^\star-\bar\pi \rangle + \max_{\pi\in\Pi} \langle \bar r, \pi-\pi^\star \rangle.
\end{align*}
Multiplying by $K$ and using the definitions $\bar\pi = \frac{1}{K}\sum_{k=1}^K \pi^k$ and $\bar r = \frac{1}{K}\sum_{k=1}^K r^k$, we can decompose $K \cdot \mathrm{DualGap}$ by adding and subtracting $\sum_{k=1}^K \langle r^k, \pi^\star - \pi^k \rangle$:
\begin{align}
    K \cdot \mathrm{DualGap} &= \max_{r\in\mathcal{R}} \sum_{k=1}^K \langle r, \pi^\star - \pi^k \rangle + \max_{\pi\in\Pi} \sum_{k=1}^K \langle r^k, \pi - \pi^\star \rangle \nonumber\\
    &= \underbrace{\left( \max_{r\in\mathcal{R}} \sum_{k=1}^K \langle r, \pi^\star - \pi^k \rangle - \sum_{k=1}^K \langle r^k, \pi^\star - \pi^k \rangle \right)}_{\text{Regret}_r} \nonumber\\
    &\quad + \underbrace{\left( \max_{\pi\in\Pi} \sum_{k=1}^K \langle r^k, \pi - \pi^\star \rangle + \sum_{k=1}^K \langle r^k, \pi^\star - \pi^k \rangle \right)}_{\text{Regret}_\pi}.
\label{eqn:dual-gap-decompose}
\end{align}
Notice that the two inner terms in the second bracket combine such that $\langle r^k, \pi - \pi^\star \rangle + \langle r^k, \pi^\star - \pi^k \rangle = \langle r^k, \pi - \pi^k \rangle$. Thus, it exactly simplifies to the $\pi$-player's regret:
\begin{align*}
    \text{Regret}_\pi = \max_{\pi\in\Pi} \sum_{k=1}^K \langle r^k, \pi - \pi^k \rangle.
\end{align*}

By Lemma~\ref{lem:omd}, the $r$-player's regret is bounded by $\mathcal{O}(R^2_{\max}B\sqrt{K})$. 

For the $\pi$-player, treating $\pi \in \Pi$ as the comparator in Lemma~\ref{lem:one-step} and taking $\eta=\beta^{-1}$, we sum the one-step descent bound over $k=1,\dots,K$:
\begin{align*}
    \sum_{k=1}^K \langle r^k, \pi - \pi^k \rangle \leq \sum_{k=1}^K \frac{R^2_{\max}}{2\beta} + \beta \sum_{k=1}^K \left[ D_{\text{KL}}(\pi(\cdot|x)\Vert\pi^{k}(\cdot|x)) - D_{\text{KL}}(\pi(\cdot|x)\Vert\pi^{k+1}(\cdot|x)) \right].
\end{align*}
Taking the maximum over $\pi \in \Pi$, bounding the telescoping sum by $D = \max_{\pi\in\Pi} D_{\text{KL}}(\pi^\star\Vert\pi)$, and selecting $\beta=\sqrt{K}/D$, we obtain:
\begin{align*}
    \text{Regret}_\pi \leq \frac{K R^2_{\max}}{2\beta} + \beta D = \mathcal{O}(D \cdot R^2_{\max}\sqrt{K}).
\end{align*}

Substituting the bounds for $\text{Regret}_r$ and $\text{Regret}_\pi$ back into \eqref{eqn:dual-gap-decompose} yields:
\begin{align*}
    K \cdot \mathrm{DualGap} \leq \mathcal{O}(R^2_{\max}B\sqrt{K}) + \mathcal{O}(D \cdot R^2_{\max}\sqrt{K}) = \mathcal{O}\left( (D+B)R^2_{\max}\sqrt{K} \right).
\end{align*}
Dividing both sides by $K$, we obtain the final bound:
\begin{align*}
    \mathrm{DualGap} \leq \mathcal{O}\left(\frac{(D+B)R^2_{\max}}{\sqrt{K}}\right),
\end{align*}
which concludes the proof.
\end{proof}
\section{Proof of Proposition~\ref{prop:bounded-reward}}
\begin{proof}[Proof of Proposition~\ref{prop:bounded-reward}]
    We first proof the boundedness of the optimal reward. This proof follows the proof of Proposition A.2 in \citet{al2023ls}. For the mixed $\chi^2$ divergence, by the definition of the variational form of Pearson $\chi^2$ divergence, we have:
    \begin{align*}
        &2 D_{\chi^2}(\pi^\star\Vert(\pi+\pi^\star)/2)\\
        &=\max_{r}~2\EE_\rho\left(\EE_{\pi^\star}[r(x,y)]-\EE_{\pi}[r(x,y)]-\frac{c}{2}\EE_{\pi^\star}[(r(x,y))^2]-\frac{c}{2}\EE_{\pi}[(r(x,y))^2]\right)\\
        &=\max_r\int_{\mathcal{X}}\int_{\mathcal{Y}}\pi^\star(y|x)\rho(x)\left(r(x,y)-\frac{c}{2}r(x,y)^2\right)-\pi(y|x)\rho(x)\left(r(x,y)+\frac{c}{2}r(x,y)^2\right)\text{d}x\text{d}y.
    \end{align*}
    For any $a,b\in\mathbb{R}^+/\{0\}$, the function $r\rightarrow a(r-\frac{c}{2}r^2)-b(r+\frac{c}{2}r^2)$ reaches its maximum at $\frac{1}{c}\frac{a-b}{a+b}$, which is bounded by $[-1/c,1/c]$. By setting $a=\rho(x)\pi^\star(y|x)$ and $b=\rho(x)\pi(y|x)$, we can obtain the closed-form of the optimal reward under mixed $\chi^2$ divergence matching:
    \begin{align*}
        r^\star(x,y)=\frac{1}{c}\frac{\pi^\star(y|x)-\pi(y|x)}{\pi^\star(y|x)+\pi(y|x)}\mathbb{I}\big(\pi^\star(y|x) \neq 0 \wedge \pi(y|x) \neq 0\big),
    \end{align*}
    which is bounded by $[-1/c,1/c]$. Furthermore, we prove the boundedness of the mixed Pearson $\chi^2$ divergence $D_{\chi^2}(\pi^\star\Vert(\pi^\star+\pi)/2)$, i.e., $0\leq 2D_{\chi^2}(\pi^\star\Vert(\pi^\star+\pi)/2)\leq \frac{1}{c}$. The proof of this result is adapted from the proof of Proposition A.3 in \citet{al2023ls}. We consider the following algebraic transformation by plugging in the optimal reward formulation:
    \begin{align*}
        &2 D_{\chi^2}(\pi^\star\Vert(\pi+\pi^\star)/2)\\
        &=\int_{\mathcal{X}}\int_{\mathcal{Y}}\pi^\star(y|x)\rho(x)\left(r^\star(x,y)-\frac{c}{2}r^\star(x,y)^2\right)-\pi(y|x)\rho(x)\left(r^\star(x,y)+\frac{c}{2}r^\star(x,y)^2\right)\text{d}x\text{d}y\\
        &=\int_{\mathcal{X}}\int_{\mathcal{Y}}\pi^\star\rho\left(\frac{1}{c}\left(\frac{\pi^\star-\pi}{\pi^\star+\pi}\right)-\frac{1}{2c}\left(\frac{\pi^\star-\pi}{\pi^\star+\pi}\right)^2\right)\\
        &\qquad-\pi\rho\left(\frac{1}{c}\left(\frac{\pi^\star-\pi}{\pi^\star+\pi}\right)+\frac{1}{2c}\left(\frac{\pi^\star-\pi}{\pi^\star+\pi}\right)^2\right)\text{d}x\text{d}y\\
        &=\int_{\mathcal{X}}\int_{\mathcal{Y}}\frac{\rho(x)}{2c}\frac{(\pi^\star(y|x)-\pi(y|x))^2}{\pi^\star(y|x)+\pi(y|x)}\text{d}x\text{d}y\\
        &=\frac{1}{2c}\EE_\rho\left(\EE_{\pi^\star}\left[\frac{\pi^\star}{\pi^\star+\pi}\right]+\EE_{\pi}\left[\frac{\pi}{\pi^\star+\pi}\right]+\EE_{\pi^\star}\left[\frac{-2\pi}{\pi^\star+\pi}\right]\right)\leq\frac{1}{c},
    \end{align*}
    which concludes the proof.
\end{proof}
\section{Proof of Proposition~\ref{prop:ne}}
\begin{proof}
    By Definition of the optimization problem in~\eqref{eqn:ail-formulation}.
\end{proof}
\section{Proof of Proposition~\ref{prop:least-square}}
\begin{proof}[Proof of Proposition~\ref{prop:least-square}]
    Consider the reward update rule:
    \begin{align*}
        ( r)^k=\text{argmax}_{ r}~\mathcal{J}( r)\coloneq\EE_{\rho}\left[\sigma(\EE_{\pi^\star}~ r(x,y)-\EE_{\pi^{k}}~ r(x,y))-\psi( r,( r)^{k-1})\right],
    \end{align*}
    where $\psi( r,( r)^{k-1})=\zeta D_{f}( r,( r)^{k-1})+c\alpha\cdot\EE_{\pi^\star}[( r(x,y))^2]+c(1-\alpha)\cdot\EE_{\pi}[( r(x,y))^2]$ is the Bregman divergence constrained convex regularizer, and $\sigma(t)=t$ is the identical link function. By simple algebraic manipulation:
    \begin{align*}
        &\mathcal{J}(r)\\
        &=\EE_{\rho}\left[\EE_{\pi^\star}~ r(x,y)-\EE_{\pi^{k}}~ r(x,y)-\psi( r,( r)^{k-1})\right]\\
        &=\EE_{\rho}\left[\EE_{\pi^\star}~ r(x,y)-\EE_{\pi^{k}}~ r(x,y)-c\alpha\cdot\EE_{\pi^\star}[( r(x,y))^2]-c(1-\alpha)\cdot\EE_{\pi^k}[( r(x,y))^2]\right]\\
        &\qquad-\EE_\rho\zeta D_{f}( r,( r)^{k-1})\\
        &=\EE_{\rho}c\alpha\left[\frac{1}{c\alpha}\EE_{\pi^\star}~ r(x,y)-\cdot\EE_{\pi^\star}[( r(x,y))^2]\right]+\EE_\rho c(1-\alpha)\left[-\EE_{\pi^k}[( r(x,y))^2]-\frac{1}{c(1-\alpha)}\EE_{\pi^{k}}~ r(x,y)\right]\\
        &\qquad-\EE_\rho\zeta D_{f}( r,( r)^{k-1})\\
        &=-\EE_{\rho,\pi^\star}~c\alpha~\left[ r(x,y)-\frac{1}{2c\alpha}\right]^2+\frac{1}{4c\alpha}-\EE_{\rho,\pi^k}~c(1-\alpha)\left[ r(x,y)+\frac{1}{2c(1-\alpha)}\right]^2+\frac{1}{4c(1-\alpha)}\\
        &\qquad-\EE_\rho\zeta D_{f}( r,( r)^{k-1}).
    \end{align*}
    Turning the maximization to minimization:
    \begin{align*}
        &\text{argmax}_{ r}~\mathcal{J}( r)\\
        &=\text{argmin}_{ r}~\mathcal{L}( r)\\
        &\coloneq\EE_{\rho,\pi^\star}~c\alpha~\left[ r(x,y)-\frac{1}{2c\alpha}\right]^2+\frac{1}{4c\alpha}+\EE_{\rho,\pi^k}~c(1-\alpha)\left[ r(x,y)+\frac{1}{2c(1-\alpha)}\right]^2+\frac{1}{4c(1-\alpha)}\\
        &\qquad+\EE_\rho\zeta D_{f}( r,( r)^{k-1}),
    \end{align*}
    which is equivalent to:
    \begin{align*}
        \text{argmin}_{ r}\EE_{\rho,\pi^\star}~\alpha~\left[ r(x,y)-\frac{1}{2c\alpha}\right]^2+\EE_{\rho,\pi^k}~(1-\alpha)\left[ r(x,y)+\frac{1}{2c(1-\alpha)}\right]^2+\EE_\rho~\zeta' D_{f}( r,( r)^{k-1}),
    \end{align*}
    where $\zeta':=\zeta/c$ and removing the constants that do not affect the learning objective. Taking $D_f(x,x')=\frac{1}{2}\Vert x-x'\Vert^2$ over both data generated by $\pi^\star$ and $\pi^k$ and plugging in the reparameterization of $ r$, i.e, $ r=\beta\log(\pi/\pi^k)$ will lead to the final objective:
    \begin{align*}
        \text{argmin}_{\pi}&\EE_{\rho,\pi^\star}~\alpha~\left[\beta\log\frac{\pi(y|x)}{\pi^k(y|x)}-\frac{1}{2c\alpha}\right]^2+\EE_{\rho,\pi^k}~(1-\alpha)\left[\beta\log\frac{\pi(y|x)}{\pi^k(y|x)}+\frac{1}{2c(1-\alpha)}\right]^2\\
        &+\EE_{\rho,\pi^\star, \pi^k}~\zeta''\left[\log\frac{\pi(y|x)}{\pi^{k}(y|x)}\right]^2.
    \end{align*}
    Taking $\zeta'':=\beta^2\zeta/c$ concludes the proof.
\end{proof}
\section{Proof of Proposition~\ref{prop:sppo}}
\begin{proof}
From Proposition~\ref{prop:sppo}, we consider the general preference alignment formulation
with the reparameterized reward $r$ defined in~\eqref{eqn:reward-mapping}. For each $x\sim\rho(x)$, we sample
$y\sim\pi^k(\cdot|x)$ and denote $w^k(x,y):=P(y\succ\pi^k\mid x)$.
The reward-player objective is
\begin{align}
    \max_{r}\;
    \EE_{\rho}\Big[
    \EE_{\pi_+^k}r(x,y)
    -
    \EE_{\pi_-^k}r(x,y)
    -
    \psi( r,r^{k-1})
    \Big].
\end{align}

Taking balanced sampling with $\alpha=0.5$, the mixed $\chi^2$ divergence regularizer without
the Bregman divergence $D_f(r,r^{k-1})$ reduces to
\begin{align}
    \phi(r,r^{k-1})
    =
    \phi(r)
    =
    \frac{c}{2}
    \EE_{\pi_+^k}\!\big[(r(x,y))^2\big]
    +
    \frac{c}{2}
    \EE_{\pi_-^k}\!\big[(r(x,y))^2\big].
\end{align}

By simple algebraic transformation,
\begin{align*}
    &
    \EE_{\pi_+^k}r(x,y)
    -
    \EE_{\pi_-^k}r(x,y)
    -
    \psi(r)
    \\
    &=
    \EE_{y\sim\pi^k(\cdot|x)}
    \big[
    w^k(x,y)r(x,y)
    -
    (1-w^k(x,y))r(x,y)
    \big]
    -
    c\EE_{y\sim\pi^k(\cdot|x)}\big[(r(x,y))^2\big]
    \\
    &=
    \EE_{y\sim\pi^k(\cdot|x)}
    \big[
    (2w^k(x,y)-1)r(x,y)
    -
    c(r(x,y))^2
    \big].
\end{align*}

Discarding constant terms that do not affect the optimization, the above
objective is equivalent to
\begin{align}
    \min_{r}\;
    \EE_{y\sim\pi^k(\cdot|x)}
    \Big[
    r(x,y)
    -
    \frac{2w^k(x,y)-1}{2c}
    \Big]^2.
    \label{eq:sppo-ls}
\end{align}

Substituting the reparameterized reward in~\eqref{eqn:reward-mapping},
\begin{align*}
r(x,y)
:=
\beta\log\frac{\pi(y|x)}{\pi^k(y|x)},
\end{align*}
the minimization in~\eqref{eq:sppo-ls} can be written as
\begin{align}
    \min_{\pi}\;
    \mathcal{L}(\pi)
    :=
    \EE_{y\sim\pi^k(\cdot|x)}
    \left[
    \log\frac{\pi(y|x)}{\pi^k(y|x)}
    -
    \frac{1}{c\beta}\Big(w^k(x,y)-\tfrac{1}{2}\Big)
    \right]^2.
\end{align}

The original SPPO algorithm~\citep{wu2024self} sets $c=1$ for the $\chi^2$
regularization, which completes the proof.
\end{proof}
\section{Proof of Proposition~\ref{prop:inpo}}
\begin{proof}
From Proposition~\ref{prop:inpo}, we consider the general preference alignment formulation
with the reparameterized reward $r$ defined in~\eqref{eqn:delta-r-inpo}. For each $x\sim\rho(x)$, we sample
$(y,y')\sim\pi^k(\cdot|x)\times\pi^k(\cdot|x)$ and define
$w^k(x,y):=P(y\succ\pi^k\mid x)$.
The reward-player objective is
\begin{align}
    \max_{r}\;
    \EE_{\rho}\Big[
    \EE_{(y,y')\sim\pi^k\times\pi^k}
    \big[(w^k(x,y)-w^k(x,y'))r(x,y,y')\big]
    -\psi(r,r^{k-1})
    \Big].
\end{align}

Taking balanced sampling with $\alpha=0.5$, the mixed $\chi^2$ divergence regularizer without
Bregman divergence $D_f(r,r^{k-1})$ reduces to
\begin{align}
    \psi(r,r^{k-1})
    =
    \psi(r)
    =
    \frac{c}{2}
    \EE_{(y,y')\sim\pi^k\times\pi^k}
    \big[r(x,y,y')^2\big].
\end{align}

By simple algebraic transformation,
\begin{align*}
    &
    \EE_{(y,y')\sim\pi^k\times\pi^k}
    \big[(w^k(x,y)-w^k(x,y'))r(x,y,y')\big]
    -\psi(r)
    \\
    &=
    -\frac{c}{2}
    \EE_{(y,y')\sim\pi^k\times\pi^k}
    \Big[
    r(x,y,y')
    -
    \frac{w^k(x,y)-w^k(x,y')}{c}
    \Big]^2
    +\text{const}.
\end{align*}

Discarding constant terms that do not affect the optimization and setting $c=1$,
the reward-player optimization is equivalent to
\begin{align}
    \min_{r}\;
    \mathcal{L}(r)
    :=
    \EE_{(y,y')\sim\pi^k\times\pi^k}
    \Big[
    r(x,y,y')
    -
    \big(w^k(x,y)-w^k(x,y')\big)
    \Big]^2.
    \label{eqn:delta-r-obj-inpo}
\end{align}

We further substitute the reward reparameterization in~\eqref{eqn:delta-r-inpo},
\begin{align*}
r(x,y,y')
=
\eta\log\frac{\pi(y|x)}{\pi(y'|x)}
-
\tau\log\frac{\pi_{\mathrm{ref}}(y|x)}{\pi_{\mathrm{ref}}(y'|x)}
-
(\eta-\tau)\log\frac{\pi^k(y|x)}{\pi^k(y'|x)}.
\end{align*}

The minimization in~\eqref{eqn:delta-r-obj-inpo} can thus be written as
\begin{align}
    \min_{\pi}\;
    \mathcal{L}(\pi)
    :=
    \EE_{(y,y')\sim\pi^k\times\pi^k}
    \left[
    h^k(\pi,x,y,y')
    -
    \frac{w^k(x,y)-w^k(x,y')}{\eta}
    \right]^2,
\end{align}
where
\begin{align*}
h^k(\pi,x,y,y')
=
\log\frac{\pi(y|x)}{\pi(y'|x)}
-
\frac{\tau}{\eta}\log\frac{\pi_{\mathrm{ref}}(y|x)}{\pi_{\mathrm{ref}}(y'|x)}
-
\frac{\eta-\tau}{\eta}\log\frac{\pi^k(y|x)}{\pi^k(y'|x)}.
\end{align*}

By Proposition~6 in \citet{zhang2024iterative}, minimizing the above objective
recovers the policy update used in INPO, which completes the proof.
\end{proof}
\section{Linear SPIN as TV Distance Minimization}
\label{sec:linear-spin-tv}
In this section, we show that linear SPIN is equivalent to minimizing the total variation (TV) distance between the model distribution and the expert data distribution. Choosing $\psi(r)=0$ if $|r|\leq R_{\max}$ and $\psi(r)=\infty$ otherwise, setting $\sigma(t)=t$,~\eqref{eqn:ail-formulation} is equivalent to:
\begin{align*}
    \max_{r\in\mathcal{R}} \min_{\pi\in\Pi}\underset{x\sim\rho(x)}{\EE}\left(\underset{y\sim\pi^\star}{\EE}r(x,y)-\underset{y\sim\pi}{\EE}r(x,y)\right)\mathbb{I}(r\leq R_{\max}),
\end{align*}
where $R_{\max}$ can be arbitrarily large since original SPIN doesn't constrain the reward magnitude explicitly. By leveraging the reward reparameterization in~\eqref{eqn:reward-mapping}:
\begin{align*}
    r(x,y)=\beta\log\left(\frac{\pi(y|x)}{\pi^k(y|x)}\right),
\end{align*}
it recovers the identical version of SPIN. By algebraic transformation:
\begin{align}
\label{eqn:linear-spin-ineq}
    \max_{r\in\mathcal{R}} \min_{\pi\in\Pi}~\mathcal{J}(r,\pi)&\coloneq\underset{x\sim\rho(x)}{\EE}\left(\underset{y\sim\pi^\star}{\EE}r(x,y)-\underset{y\sim\pi}{\EE}r(x,y)\right)\mathbb{I}(r\leq R_{\max})\nonumber\\
    &=\underset{x\sim\rho(x)}{\EE}\langle r(x,\cdot),\pi^\star(\cdot|x)-\pi(\cdot|x)\rangle\mathbb{I}(r\leq R_{\max})\nonumber\\
    &\leq R_{\max}\EE_{\rho}\vert\pi^\star(\cdot|x)-\pi(\cdot|x)\vert.
\end{align}
Consider $r^\star(x,y)$ as the optimal reward solution for reward maximization part, by Assumption~\ref{ass:realize}:
\begin{align*}
    r^\star(x,y)=R_{\max}\cdot\text{sgn}(\pi^\star(y|x)-\pi(y|x)).
\end{align*}
Therefore, the inequality in~\eqref{eqn:linear-spin-ineq} reaches equal when the optimal reward is reached. In this case,
\begin{align*}
    \max_{r\in\mathcal{R}} \min_{\pi\in\Pi}~\mathcal{J}(r,\pi)=\min_{\pi\in\Pi}2R_{\max}\cdot \EE_{x\sim\rho}~D_{\text{TV}}(\pi(\cdot|x),\pi^\star(\cdot|x))
\end{align*}
holds because $D_{\text{TV}}(\mu,\nu)=\frac{1}{2}\Vert\mu-\nu\Vert_1$. This is a Total Variation distance minimization.
\section{Non-Linear SPIN as KL Divergence Minimization}
\label{sec:non-linear-spin}
We begin by the following lemma for contraction:
\begin{lemma}[KL contraction toward $\pi^{\star}$]
Let $\beta\ge 1$ and $\alpha\coloneqq 1/\beta\in(0,1]$, and define the update
\begin{equation}
\label{eq:update}
\pi^{k+1}(y|x)
\;\propto\;
\pi^k(y|x)\Big(\frac{\pi^\star(y|x)}{\pi^k(y|x)}\Big)^{\alpha}
\;=\;
\pi^k(y|x)^{1-\alpha} \pi^\star(y|x)^{\alpha}.
\end{equation}
Then the reverse KL to data contracts geometrically:
\begin{equation*}
D_{\text{KL}}\big(\pi^\star\|\pi^{k+1}\big)
\;\le\;
(1-\alpha)\,D_{\text{KL}}\big(\pi^\star\|\pi^{k}\big)
\;=\;
\Big(1-\frac{1}{\beta}\Big)D_{\text{KL}}\big(\pi^\star\|\pi^{k}\big).
\end{equation*}
Consequently,
\begin{equation*}
D_{\text{KL}}\big(\pi^\star\|\pi^{k}\big)\;\le\;\Big(1-\frac{1}{\beta}\Big)^k D_{\text{KL}}\big(\pi^\star\|\pi^{\mathrm{ref}}\big)\xrightarrow[k\to\infty]{}0.
\end{equation*}
\label{lem:kl-contraction}
\end{lemma}

\begin{proof}
We begin by explicitly writing the normalized update rule. Let $Z(x)$ be the normalization constant (partition function) for the update in~\eqref{eq:update}:
\begin{equation*}
    Z(x) = \sum_y \pi^k(y|x)^{1-\alpha} \pi^\star(y|x)^{\alpha} = \mathbb{E}_{y \sim \pi^k(\cdot|x)} \left[ \left( \frac{\pi^\star(y|x)}{\pi^k(y|x)} \right)^{\alpha} \right].
\end{equation*}
The normalized policy is then given by:
\begin{equation*}
    \pi^{k+1}(y|x) = \frac{1}{Z(x)} \pi^k(y|x)^{1-\alpha} \pi^\star(y|x)^{\alpha}.
\end{equation*}
We now expand the KL divergence $D_{\text{KL}}(\pi^\star \| \pi^{k+1})$:
\begin{align}
    D_{\text{KL}}(\pi^\star \| \pi^{k+1}) &= \mathbb{E}_{y \sim \pi^\star} \left[ \log \frac{\pi^\star(y|x)}{\pi^{k+1}(y|x)} \right]\nonumber \\
    &= \mathbb{E}_{y \sim \pi^\star} \left[ \log \pi^\star(y|x) - \log \left( \frac{\pi^k(y|x)^{1-\alpha} \pi^\star(y|x)^{\alpha}}{Z(x)} \right) \right]\nonumber \\
    &= \mathbb{E}_{y \sim \pi^\star} \left[ \log \pi^\star(y|x) - (1-\alpha)\log \pi^k(y|x) - \alpha \log \pi^\star(y|x) + \log Z(x) \right] \nonumber\\
    &= (1-\alpha) \mathbb{E}_{y \sim \pi^\star} \left[ \log \pi^\star(y|x) - \log \pi^k(y|x) \right] + \log Z(x)\nonumber \\
    &= (1-\alpha) D_{\text{KL}}(\pi^\star \| \pi^k) + \log Z(x). \label{eq:kl_expansion}
\end{align}
To bound the remainder term $\log Z(x)$, we invoke Jensen's inequality. Since $\beta \ge 1$, we have $\alpha = 1/\beta \in (0, 1]$. Consequently, the function $f(t) = t^\alpha$ is concave. Applying Jensen's inequality to the expectation of the likelihood ratio under $\pi^k$:
\begin{align*}
    Z(x) &= \mathbb{E}_{y \sim \pi^k} \left[ \left( \frac{\pi^\star(y|x)}{\pi^k(y|x)} \right)^{\alpha} \right] \\
    &\le \left( \mathbb{E}_{y \sim \pi^k} \left[ \frac{\pi^\star(y|x)}{\pi^k(y|x)} \right] \right)^{\alpha} \\
    &= \left( \sum_y \pi^k(y|x) \frac{\pi^\star(y|x)}{\pi^k(y|x)} \right)^{\alpha} \\
    &= \left( \sum_y \pi^\star(y|x) \right)^{\alpha} = 1^\alpha = 1.
\end{align*}
Thus, $Z(x) \le 1$, which implies $\log Z(x) \le 0$. Substituting this inequality back into~\eqref{eq:kl_expansion}, we obtain:
\begin{equation*}
    D_{\text{KL}}(\pi^\star \| \pi^{k+1}) \le (1-\alpha) D_{\text{KL}}(\pi^\star \| \pi^k).
\end{equation*}
Substituting $\alpha = 1/\beta$ yields the geometric contraction:
\begin{equation*}
    D_{\text{KL}}(\pi^\star \| \pi^{k+1}) \le \left(1 - \frac{1}{\beta}\right) D_{\text{KL}}(\pi^\star \| \pi^k).
\end{equation*}
Applying this inequality recursively $k$ times leads to the final convergence rate:
\begin{equation*}
    D_{\text{KL}}(\pi^\star \| \pi^k) \le \left(1 - \frac{1}{\beta}\right)^k D_{\text{KL}}(\pi^\star \| \pi^{\mathrm{ref}}).
\end{equation*}
This completes the proof.
\end{proof}
Consider the choice of logistic link function $\sigma(t)=-\log(1+\exp(-t))$ in~\eqref{eqn:two-stage-formulation}, $\psi(r)=\infty \cdot \mathbf{1}[|r|_{\infty} > R_{\max}]$ without the Bregman constraint over $r$, it recovers the original SPIN \citep{chen2024self} objective with non-identical link function. SPIN has proved the following lemma:
\begin{lemma}[Theorem 5.4 in \citealt{chen2024self}]
    Consider the choice of logistic link function in SPIN. Suppose $\pi^{k+1}$ is the global minimum for the SPIN objective at iteration $k$, then the opponent player at iteration $k+1$ satisfies:
    \begin{equation*}
        \pi^{k+1}(y|x)\propto\pi^k(y|x)\left(\pi^\star(y|x)/\pi^k(y|x)\right)^{1/\beta}.
    \end{equation*}
    \label{lem:update}
\end{lemma}
By first applying Lemma~\ref{lem:update} for SPIN objective then applying Lemma~\ref{lem:kl-contraction}, we can prove that SPIN with non-identical link function contracts under KL divergence between $\pi^\star$ and $\pi^k$ for multiple iterations.
\section{Iterative DPO as KL Divergence Minimization}
\label{sec:iterative_dpo}

In this section, we analyze the theoretical properties of Iterative DPO \citep{tu2025enhancing,wu2024self}. We demonstrate that iteratively solving the KL-regularized reinforcement learning objective—using the current policy as the reference—constitutes a contraction mapping that minimizes the KL divergence between the current policy and the optimal policy implied by the preference oracle. Let $\mathcal{X}$ be the input space and $\mathcal{Y}$ be the output space. We assume access to a preference oracle $P(y_1 \succ y_2 | x)$ which adheres to the Bradley-Terry (BT) model. Under the BT model, the preference probability is determined by a latent reward function $r^\star(x,y)$:
\begin{equation*}
    P(y_1 \succ y_2 | x) = \frac{\exp(r^\star(x, y_1))}{\exp(r^\star(x, y_1)) + \exp(r^\star(x, y_2))} = \sigma(r^\star(x, y_1) - r^\star(x, y_2)).
\end{equation*}
Let $\pi^\star$ denote the optimal policy that perfectly captures the underlying reward structure, i.e., $\pi^\star(y|x) \propto \exp(r^\star(x,y))$.

We consider an iterative setting where, at step $k$, we optimize a policy $\pi$ against a reference policy $\pi^k$ (the policy from the previous iteration). The objective is to maximize the expected reward subject to a KL-divergence constraint:
\begin{equation} \label{eq:objective}
    \max_{\pi} \mathcal{J}_k(\pi) = \mathbb{E}_{y \sim \pi(\cdot|x)} [r^\star(x, y)] - \beta D_{\text{KL}}(\pi(\cdot|x) || \pi^k(\cdot|x)).
\end{equation}

\begin{lemma} \label{prop:update_rule}
Given the reference policy $\pi^k$ and the preference oracle $P$, the optimal policy $\pi^{k+1}$ maximizing the objective in~\eqref{eq:objective} is given by:
\begin{equation*}
    \pi^{k+1}(y|x) \propto \pi^k(y|x) \left( \frac{P(y \succ y_{\mathrm{ref}} | x)}{1 - P(y \succ y_{\mathrm{ref}} | x)} \right)^{1/\beta}
\end{equation*}
where preference estimates are taken relative to a baseline $y_{\text{ref}}$ drawn from $\pi^k$.
\end{lemma}

\begin{proof}
The optimization problem in~\eqref{eq:objective} has a well-known closed-form solution given by the Boltzmann distribution:
\begin{equation} \label{eq:boltzmann_update}
    \pi^{k+1}(y|x) = \frac{1}{Z_k(x)} \pi^k(y|x) \exp\left( \frac{r^\star(x,y)}{\beta} \right).
\end{equation}
To express the reward $r^\star(x,y)$ in terms of the oracle, we utilize the Bradley-Terry relationship. The odds of preferring $y$ over a reference output $y_{\text{ref}}$ are:
\begin{equation*}
    \frac{P(y \succ y_{\text{ref}} | x)}{P(y_{\text{ref}} \succ y | x)} = \frac{e^{r^\star(x,y)}}{e^{r^\star(x,y_{\text{ref}})}} = \exp(r^\star(x,y) - r^\star(x,y_{\text{ref}})).
\end{equation*}
Taking the logarithm yields the reward function (up to a constant shift $r^\star(x, y_{\text{ref}})$ which is absorbed into the partition function):
\begin{equation*}
    r^\star(x,y) = \log \left( \frac{P(y \succ y_{\text{ref}} | x)}{1 - P(y \succ y_{\text{ref}} | x)} \right) + C.
\end{equation*}
Substituting this expression back into~\eqref{eq:boltzmann_update} yields the proposition:
\begin{equation*}
    \pi^{k+1}(y|x) \propto \pi^k(y|x) \left( \frac{P(y \succ y_{\text{ref}} | x)}{1 - P(y \succ y_{\text{ref}} | x)} \right)^{1/\beta}.
\end{equation*}
\end{proof}
We now show that this iterative process monotonically reduces the distance to the optimal policy $\pi^\star$.

\begin{proposition}[Contraction for Iterative DPO]
Let $\pi^\star$ be the global optimal policy implied by the reward oracle. The sequence of policies $\{\pi^k\}_{k=0}^{\infty}$ generated by the update rule in Proposition \ref{prop:update_rule} satisfies the following contraction inequality:
\begin{equation*}
    D_{\text{KL}}(\pi^\star || \pi^{k+1}) \leq D_{\text{KL}}(\pi^\star || \pi^k) - D_{\text{KL}}(\pi^{k+1} || \pi^k).
\end{equation*}
Consequently, $D_{\text{KL}}(\pi^\star || \pi^{k+1}) < D_{\text{KL}}(\pi^\star || \pi^k)$ for all non-trivial updates ($\pi^{k+1} \neq \pi^k$).
\end{proposition}

\begin{proof}
We expand the KL divergence term $D_{\text{KL}}(\pi^\star || \pi^{k+1})$:
\begin{align}
    D_{\text{KL}}(\pi^\star || \pi^{k+1}) &= \sum_y \pi^\star(y) \log \frac{\pi^\star(y)}{\pi^{k+1}(y)} \nonumber \\
    &= \sum_y \pi^\star(y) \left[ \log \pi^\star(y) - \log \left( \frac{\pi^k(y) \exp(r^\star(y)/\beta)}{Z_k} \right) \right] \nonumber \\
    &= \sum_y \pi^\star(y) \log \frac{\pi^\star(y)}{\pi^k(y)} - \frac{1}{\beta} \mathbb{E}_{y \sim \pi^\star}[r^\star(y)] + \log Z_k \nonumber \\
    &= D_{\text{KL}}(\pi^\star || \pi^k) - \frac{1}{\beta} \mathbb{E}_{\pi^\star}[r^\star] + \log Z_k. \label{eq:expansion}
\end{align}
Next, we relate the log-partition function $\log Z_k$ to the KL divergence between steps. By definition:
\begin{equation*}
    D_{\text{KL}}(\pi^{k+1} || \pi^k) = \mathbb{E}_{\pi^{k+1}} \left[ \log \frac{\pi^{k+1}}{\pi^k} \right] = \mathbb{E}_{\pi^{k+1}} \left[ \frac{r^\star}{\beta} - \log Z_k \right] = \frac{1}{\beta} \mathbb{E}_{\pi^{k+1}}[r^\star] - \log Z_k.
\end{equation*}
Solving for $\log Z_k$:
\begin{equation} \label{eq:zk_sub}
    \log Z_k = \frac{1}{\beta} \mathbb{E}_{\pi^{k+1}}[r^\star] - D_{\text{KL}}(\pi^{k+1} || \pi^k).
\end{equation}
Substituting~\eqref{eq:zk_sub} into~\eqref{eq:expansion}:
\begin{equation*}
    D_{\text{KL}}(\pi^\star || \pi^{k+1}) = D_{\text{KL}}(\pi^\star || \pi^k) - D_{\text{KL}}(\pi^{k+1} || \pi^k) - \frac{1}{\beta} \underbrace{\left( \mathbb{E}_{\pi^\star}[r^\star] - \mathbb{E}_{\pi^{k+1}}[r^\star] \right)}_{\Delta \geq 0}.
\end{equation*}
Since $\pi^\star$ maximizes the expected reward, the term $\Delta$ is non-negative. Since $D_{\text{KL}}$ is non-negative, the inequality holds strictly, proving contraction towards the oracle distribution.
\end{proof}
\begin{table}[t]
    \centering
    \begin{tabular}{c|c}
    \toprule
        \textbf{Hyperparameters} & \textbf{Value} \\
    \midrule
        $\beta$ & $1e-3$\\
        $\zeta$ & $1e-3$\\
        $\alpha$ & $0.5$\\
        $c$ & $2$\\
        Batch Size & 32 \\
        Optimizer & AdamW \\
        Learning Rate & $5e-7$\\
        Learning Rate Scheduler & Linear\\
        Warm-up Ratio & $0.1$\\
        Epochs per Iteration & $3$\\
        Generation Length & 256\\
    \bottomrule
    \end{tabular}
    \caption{\textbf{Hyperparameters.} This table lists the hyperparameters used for the algorithm, training, and generation.}
    \label{tab:hyperparams}
\end{table}
\section{Experimental Settings and Implementation Details}
\label{sec:implementation}
\subsection{Hyperparameters}
In this section, we describe in detail the hyperparameters used in the practical implementation of our method. We adopt a single fixed set of hyperparameters across all self play iterations and model architectures. Table~\ref{tab:hyperparams} summarizes the hyperparameters used in our algorithm, including the values of $\alpha$, $\beta$, $\zeta$, and $c$, as well as the training hyperparameters such as batch size, optimizer, learning rate, and scheduler, and the generation configuration including generation length.
\subsection{Experimental Settings}
\paragraph{Dataset Preparation} For the first iteration, we use 50k examples subsampled from the UltraChat SFT dataset \citep{Ding2023EnhancingCL}. For each prompt $x$, we construct $(x, y, y')$ tuples, where $y$ is the reference response from the dataset and $y'$ is the response generated by the model. The pairs $(x, y)$ are included in $\mathcal{D}^\star$, while $(x, y')$ form $\mathcal{D}^0$ for training. Starting from the second iteration, we train using both our method and SPIN \citep{chen2024self} on the datasets from the two most recent iterations, resulting in a combined dataset of 100k examples.

\paragraph{Baseline Methods} For the SFT baseline, we reproduce supervised fine tuning by directly applying maximum likelihood estimation on $\mathcal{D}^\star$. For the SPIN baseline \citep{chen2024self}, we use the authors’ official implementation with their recommended hyperparameter settings.

\paragraph{Evaluation Metrics}
We evaluate performance on Arc Challenge using a 25-shot setting with \texttt{acc\_norm} as the evaluation metric. For MMLU, we use a 5-shot setting and report \texttt{acc}. For HellaSwag, we adopt a 10-shot evaluation with \texttt{acc\_norm}. Finally, for WinoGrande, we use a 5-shot setting and report \texttt{acc}.

\section{Ablation Studies}
\label{sec:ablation}
\noindent \textbf{Ablation on Hyperparameter $c$.} We conduct an ablation study on the hyperparameter $c$, which controls the reward targets $r_{\max}$ and $r_{\min}$ in the least-squares regression objective. A larger value of $c$ results in a smaller margin between $r_{\max}$ and $r_{\min}$, whereas a smaller value of $c$ induces a larger margin and higher reward magnitude. We vary $c$ to examine the effect of this trade-off.

In our main experiments, we set $c=2$, corresponding to $r_{\max}=0.5$ and $r_{\min}=-0.5$. When $c$ is reduced to $0.5$, the resulting larger reward magnitude leads to degraded performance, which is consistent with the theoretical predictions in Sec~\ref{sec:theoretical-analysis}. Conversely, setting $c=8$ yields a small reward magnitude and a narrow margin, potentially limiting the ability to discriminate between data generated by the expert policy $\pi^\star$ and the previous-iteration policy $\pi^k$.

Figure~\ref{fig:ablation-c} summarizes the results for these settings, reporting the mean performance of the Qwen-3 4B model across four evaluation benchmarks (Arc-Challenge, MMLU, HellaSwag, WinoGrande) over three self-play iterations (with iteration 0 corresponding to the base model).

\noindent \textbf{Ablation on $D_f(r,r^{k-1})$ Constraint on Reward.} Our method includes an additional regularization term that constrains the deviation between successive policies $\pi^{k+1}$ and $\pi^k$, inducing a mirror descent structure. We ablate this component by removing the final term in~\eqref{eqn:chi-squared-self-play}. The results show that omitting this regularizer leads to noticeable performance degradation, empirically demonstrating its effectiveness. We report results across four benchmarks in Table~\ref{tab:ablation-regularizer} after three self-play iterations on the Qwen-3-4B model.

\begin{table}[t]
    \centering
    \begin{tabular}{c|cc}
    \toprule
        \textbf{Scores (3 Iters)} & \textbf{w/ Regularizer} & \textbf{w/o Regularizer} \\
        \midrule
        Arc-Challenge & 57.11 & 56.87 \\
        MMLU & 68.83 & 68.79 \\
        HellaSwag & 71.92 & 70.05 \\
        WinoGrande & 68.82 & 68.43 \\
        \bottomrule
    \end{tabular}
    \caption{\textbf{Ablation on the Reward Constraint.} We conduct an ablation study on the reward regularizer that enforces mirror descent on the reward player. The results demonstrate that this regularization improves self-play performance. All results are reported after three self-play iterations on the Qwen-3-4B model.}
    \label{tab:ablation-regularizer}
\end{table}

\begin{figure}[t]
\centering
\includegraphics[width=0.5\linewidth]{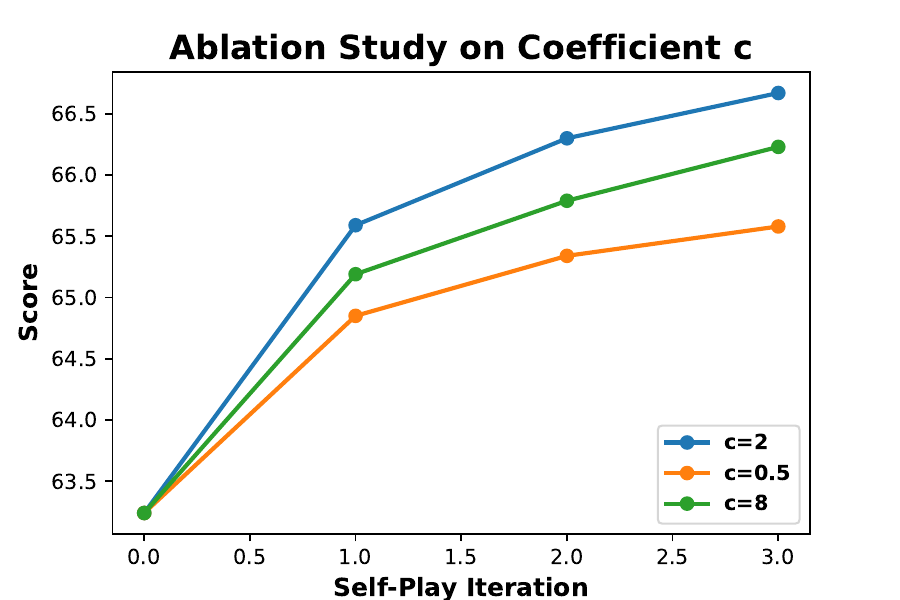}
\caption{\textbf{Ablation on Hyperparameter $c$.} We evaluate the impact of the hyperparameter $c$ by setting $c \in \{0.5, 2, 8\}$ and examining its effect on the self-play performance of our method. Performance is measured as the mean score across the four benchmarks used in the main experiments. We observe that both overly small and overly large values of $c$ lead to performance degradation, highlighting the importance of an appropriate balance in reward scaling. \looseness=-1}
\label{fig:ablation-c}
\end{figure}

\section{Additional Results on Mathematical Capabilities}
\label{sec:math}
To evaluate the mathematical reasoning capabilities of models aligned with our method, we conduct an additional evaluation on GSM8K~\citep{cobbe2021gsm8k}. We compare the aligned Qwen3-4B and Mistral-7B models over three self-play iterations against SFT, SPIN~\citep{chen2024self} and SPACE \citep{wang2026space}. Results are shown in Table~\ref{tab:gsm8k-results}.

\begin{table}[t]
\centering
\small
\setlength{\tabcolsep}{6pt}
\renewcommand{\arraystretch}{1.1}
\begin{tabular}{@{}lccccc@{}}
\toprule
\textbf{Model} & \textbf{Base} & \textbf{SFT} & \textbf{SPIN (3 Iters)} & \textbf{SPACE (3 Iters)} & \textbf{SPIF (3 Iters)} \\
\midrule
Qwen3-4B   & 84.91 & 85.67 & 86.31 & 87.24 & \textbf{88.55} \\
Mistral-7B & 33.36 & 36.94 & 39.13 & 39.52 & \textbf{40.44} \\
\bottomrule
\end{tabular}
\caption{\textbf{GSM8K Results.}
We report GSM8K accuracy for the base model, SFT, SPIN, SPACE and SPIF methods after three self-play iterations. Results are reported with five random seeds.}
\label{tab:gsm8k-results}
\end{table}

\section{Results with Larger Number of Iterations}
\label{sec:six-iter}
The main result in this paper conducts the self-play algorithm in three iterations. In this section, we demonstrate the result over six iterations on the average scores on four baselines (Arc Challenge \citep{allenai:arc}, MMLU \citep{Hendrycks2020MeasuringMM}, HellaSwag \citep{Zellers2019HellaSwagCA}, and WinoGrande \citep{Sakaguchi2019WinoGrande}) with Qwen3 4B model. Results show that the performance generally saturates as the number of iteration increases during the self-play, showing that the self-play methods cannot bootstrap infinitely and has an upper limit, aligns with the theoretical results of our adversarial imitation learning theoretical framework.

\begin{figure}[h]
    \centering
    \includegraphics[width=0.5\linewidth]{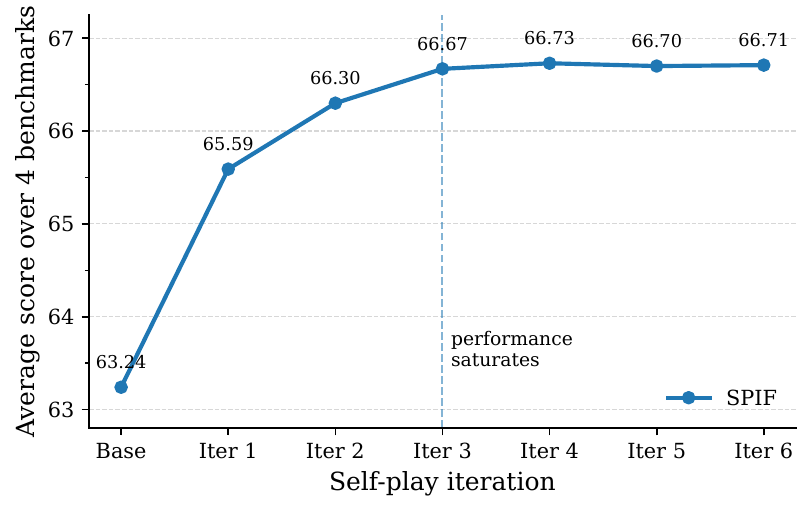}
    \caption{\textbf{Saturation of SPIF over self-play iterations.}
    We report the average score over four benchmarks for Qwen3-4B across six SPIF iterations.
    The performance improves substantially in the first few iterations and saturates after Iteration 3,
    consistent with the interpretation of SPIF as an imitation learning method.}
    \label{fig:spif-saturation}
\end{figure}

\section{Limitations}

Our theoretical analysis relies on realizability assumptions. While such assumptions are standard in theoretical analyses, they may not always hold in practical LLM finetuning scenarios, where the reward and policy classes can be misspecified. In addition, our convergence result establishes an average-iterate guarantee, whereas practical self-play finetuning evaluates the last-iterate model produced by training. Establishing last-iterate convergence would require stronger assumptions on the game structure or optimization dynamics, which we leave for future work.

\section{Generated Samples}

\begin{table}[h]
    \centering
    \begin{tabular}{p{0.18\linewidth}|p{0.75\linewidth}}
    \toprule
        Prompt & What is the best place to try traditional Indian sweets in Jaipur?\\
    \midrule
    Ground Truth & One of the best places to try traditional Indian sweets in Jaipur is Laxmi Misthan Bhandar (LMB) in Johari Bazaar. They have been serving authentic Rajasthani sweets for over 300 years and are known for their decadent desserts like Ghewar, Rasgulla, and Mawa Kachori. Their ambiance and range of sweets make for a memorable experience.\\
    \midrule
       Base Model  & The best place to try traditional Indian sweets in Jaipur is Kishore Ganj. This bustling market area is famous for its wide variety of street food, including a wide range of traditional Indian sweets like Gajar ka Halwa, Jalebi, Kheer, Barfi, Ladoo, and Mithai.
       You can also find sweet shops (called Mithai Mandis) along the streets of Kishore Ganj, where local vendors offer a variety of sweet treats. For a more authentic experience, you can visit Kishore Ganj Bazaar, which is a popular spot for both locals and tourists to enjoy the local flavors and street food culture of Jaipur.
       Other nearby areas like Rajouri Garden and Anand Niketan also have good options for traditional Indian sweets, but Kishore Ganj is widely regarded as the best place to try them in Jaipur.\\
    \midrule
       Iteration 1 & The best place to try traditional Indian sweets in Jaipur is the famous Jaipur Sweet Shop, located on the corner of Bapu Road and Jai Singh Road. This shop is known for its wide variety of sweet dishes, including laddoos, jalebis, kheer, and much more. The shop has been in operation for over 100 years and is a favorite among locals and tourists alike.\\
    \midrule
       Iteration 2 &  The best place to try traditional Indian sweets in Jaipur is the Chandni Chowk Bazaar.\\
    \midrule
       Iteration 3 & The best place to try traditional Indian sweets in Jaipur is Laxmi Mishtan Bhandar, located in the heart of the city.\\
    \bottomrule
    \end{tabular}
    \caption{\textbf{Generated Samples.} We present generated samples from the Qwen-3 4B model across different self-play finetuning iterations for a single prompt with our proposed method.}
    \label{tab:generated-samples}
\end{table}

\end{document}